\documentclass{article}

\usepackage[preprint]{neurips_2026}

\usepackage[utf8]{inputenc} 
\usepackage[T1]{fontenc}    
\usepackage{hyperref}       
\usepackage{url}            
\usepackage{booktabs}       
\usepackage{amsfonts}       
\usepackage{nicefrac}       
\usepackage{microtype}      
\usepackage{xcolor}         
\usepackage{amsthm}
\usepackage{amsmath,amssymb}
\usepackage{graphicx}

\usepackage{comment}

\newtheorem{theorem}{Theorem}

\newtheorem{lemma}{Lemma}

\newtheorem{assumption}{Assumption}

\newcommand{\NAMEA}{\textsc{Q-ANCHOR}}

\usepackage{algorithm}
\usepackage{algpseudocode}

\title{Q-ANCHOR: Federated Quantum Learning with ZNE-guided Correction}

\author{%
  Hoang M. Ngo\\
  Department of Computer \& Information Science \& Engineering\\
  University of Florida\\
  \And
  Quan Nguyen \\
  Department of Computer \& Information Science \& Engineering\\
  University of Florida\\
  \And
  Wanli Xing \\
  Frost Institute for Data Science and Computing\\
  University of Miami\\
  \And
  My T. Thai \\
  Department of Computer \& Information Science \& Engineering\\
  University of Florida\\
}

\begin{document}

\maketitle

\begin{abstract}
  Quantum Federated Learning (QFL) offers a promising framework to train quantum models across distributed clients while keeping data strictly local. Due to its simplicity and low communication overhead, Federated Averaging (FedAvg) is the standard aggregation choice in QFL literature. However, deploying QFL on practical hardware exposes a severe double-drift phenomenon: the global model is simultaneously derailed by client drift from non-IID data and hardware bias from noisy quantum gradient estimates. In this work, we first analyze the convergence of FedAvg under these realistic conditions, mathematically demonstrating that quantum hardware bias creates a persistent error floor that standard averaging cannot correct. To overcome this limitation, we propose \NAMEA, a quantum-aware federated aggregation architecture that anchors server updates with zero-noise extrapolation while applying stateful client correction to suppress both client drift and hardware-induced bias. Our convergence theory proves that \NAMEA\ successfully mitigates classical client drift while actively reducing the hardware-bias floor. Experimental results demonstrate that \NAMEA\ achieves significantly more stable training than conventional FL baselines.
\end{abstract}

\section{Introduction}
\label{sec:introduction}
With the rapid maturation of quantum computing, Quantum Machine Learning (QML) has emerged as a transformative paradigm, which demonstrates significant potential to outperform classical models in solving complex computational tasks~\citep{Havlicek2019_QML}. However, to translate this theoretical potential into practical, large-scale deployment, QML must transcend the limitations of centralized processing. Federated Learning (FL) provides a prominent distributed framework that enables collaborative model training across multiple clients while strictly preserving the privacy of localized data~\citep{McMahan2017_originalFL}. The natural integration of these two technologies has catalyzed the rapidly growing field of Quantum Federated Learning (QFL)~\citep{Chen2021_QFLrelated,Xia2021_QFLrelated,Chehimi2022_relatedQFL}. By harnessing both the enhanced representational capacity of quantum circuits and the robust privacy guarantees of decentralized training, QFL is becoming increasingly vital for deploying models in highly sensitive domains, such as drug discovery~\citep{SAGGI2024_QFL_Drug,MESSINIS2025_QFL_Drug}, medical image classification~\citep{Lusnig2024_QFLrelated} and security in smart grids~\citep{Ren2024_QFLrelated}.

In QFL, federated optimization plays a central role because the global model must be trained from local updates without directly accessing private client data. Consequently, several recent studies have attempted to improve the optimization and communication efficiency of QFL beyond basic parameter averaging. For example, Qi et al.~\citep{Qi2023_QFLrelated} replaced standard stochastic-gradient updates with Federated Quantum Natural Gradient Descent to exploit the geometry of variational quantum circuits. Zhao~\citep{Zhao2023_QFLrelated} addressed non-IID data through a one-shot QFL construction based on quantum-channel decomposition and local density estimators, while Han et al.~\citep{Han2026_relatedQFL} proposed a layerwise aggregation framework for clients with heterogeneous quantum circuit depths. However, these approaches do not fully address the coupled challenges that arise in realistic iterative QFL. Specifically, Zhao's one-shot construction does not model accumulated client drift over multiple communication rounds, whereas Han et al.'s layerwise aggregation does not account for hardware-induced bias in quantum gradient estimation. More broadly, existing QFL optimization methods still lack a unified treatment of the two fundamental forces that distort local updates in practice: iterative client drift from federated heterogeneity and quantum stochasticity from Noisy Intermediate-Scale Quantum (NISQ)~\citep{Preskill2018}.

This gap reveals a fundamental \emph{double-drift} issue in realistic QFL. On the federated side, non-IID data cause each client to optimize a different local objective, so repeated local updates can drift away from the global descent direction (known as \emph{client drift}). On the quantum side, noisy gradient estimation of parameterized quantum circuits (PQCs) introduces additional sources of update distortion, including stochastic variance from finite-shot measurements and systematic bias from hardware imperfections (known as \emph{hardware drift}). This double-drift phenomenon makes QFL optimization fundamentally different from classical FL, where the gradient computation is unbiased, and from centralized QML, where federated client heterogeneity is absent. Nevertheless, existing QFL frameworks still inherit the Federated Averaging (FedAvg) aggregation mechanism from classical FL~\citep{McMahan2017_originalFL}, where clients independently perform local optimization and the server averages their resulting updates. While this strategy is simple and effective in standard FL settings, it is not designed to correct the coupled update distortions induced by both non-IID federated data and noisy quantum gradient estimation. To date, no existing QFL work provides an aggregation mechanism that jointly corrects classical client drift and quantum-induced gradient bias over multiple communication rounds.

\noindent \textbf{Contributions.} The key contributions and insights of this work can be highlighted as follows:

\noindent \textbf{(i) FedAvg convergence analysis for QFL.}
We first provide a theoretical convergence analysis of FedAvg in QFL, where local gradients are estimated from PQCs under finite-shot measurements and noisy quantum execution. This analysis characterizes how the standard FedAvg update behaves in realistic quantum settings and reveals that its convergence can be degraded by the \emph{double-drift} effect. 

\noindent \textbf{(ii) Quantum Federated ZNE-Anchored Control (\NAMEA).}
We then propose \NAMEA, a quantum-aware control-variate aggregation framework for QFL. Unlike FedAvg, \NAMEA\ maintains stateful correction terms to counteract client drift over multiple communication rounds. To address the quantum side of the double-drift problem, \NAMEA\ further anchors the server-side correction using Zero-Noise Extrapolation (ZNE), thereby reducing the systematic bias introduced by noisy quantum hardware. This co-design enables \NAMEA\ to simultaneously correct federated statistical drift and quantum hardware-induced gradient bias.

\noindent \textbf{(iii) \NAMEA\ convergence analysis for QFL.}
We establish a convergence analysis for \NAMEA\ under non-IID client data, finite-shot gradient estimation, and hardware-induced quantum bias. Our analysis shows how the control-variate mechanism reduces accumulated client drift across communication rounds, while the ZNE-anchored correction mitigates the hardware-induced bias. 

\noindent \textbf{(iv) Empirical experiments.}
We empirically evaluate \NAMEA\ against representative baselines, including FedAvg, which is commonly adopted in QFL, and SCAFFOLD which is a popular baseline in classical FL for correcting client drift. Experiments are conducted under heterogeneous federated settings to assess the impact of non-IID data, finite-shot gradient estimation, and noisy quantum execution. The results demonstrate that \NAMEA\ consistently improves training stability and test performance by jointly correcting statistical client drift and hardware-induced quantum bias.

\section{Preliminaries}

This section establishes the foundations of QFL. We first formalize the standard classical federated optimization framework. We then introduce the core mechanics of quantum computation and detail how the physical evaluation of quantum gradients fundamentally diverges from classical optimization paradigms. More details can be found in Appendix~\ref{appendix:prelim}.

\subsection{Classical Federated Optimization and Client Drift}

In classical FL, $N$ clients collaboratively optimize
\begin{align}
\min_{x\in\mathbb{R}^d} f(x)
:= \frac{1}{N}\sum_{i=1}^N f_i(x),
\qquad
f_i(x):=\mathbb{E}_{\xi_i\sim\mathcal{D}_i}[\ell(x;\xi_i)] ,
\end{align}
without centralizing local data. FedAvg~\citep{McMahan2017_originalFL} reduces communication by letting sampled clients perform multiple local stochastic-gradient steps before server aggregation.

In classical FL, mini-batch gradients are typically unbiased local estimators,
\begin{align}
\mathbb{E}_{\xi_i}[g_i(x;\xi_i)] = \nabla f_i(x),
\end{align}
so the main optimization difficulty comes from non-IID data, that is, $\nabla f_i(x)$ may differ from the global gradient $\nabla f(x)$. Repeated local updates can then move clients toward different local objectives, causing \emph{client drift}. QFL inherits this federated client drift, but further introduces an additional \emph{hardware drift} caused by noisy gradient estimation on quantum devices, as discussed next.

\subsection{Quantum Federated Learning and Hardware Drift}

In QFL, local models are typically implemented by parameterized quantum circuits (PQCs). Given trainable parameters $x\in\mathbb{R}^d$, a PQC prepares a quantum state
$|\psi(x)\rangle = U(x)|0\rangle^{\otimes n}$ and defines the learning objective through the expectation value of an observable $O$, which is $f(x) = \langle \psi(x)|O|\psi(x)\rangle$. For standard Pauli-rotation gates, gradients can be evaluated using the parameter-shift rule~\citep{Mitarai2018_PSR,Wierichs2022_PSR}:
\begin{align*}
\frac{\partial f(x)}{\partial x_j}
=
\frac{1}{2}
\left[
f\left(x+\frac{\pi}{2}e_j\right)
-
f\left(x-\frac{\pi}{2}e_j\right)
\right],
\end{align*}
where $e_j$ is the $j$-th standard basis vector. Thus, quantum gradient estimation relies on evaluating shifted quantum circuits and estimating their expectation values from measurement outcomes.


On near-term quantum hardware, this gradient estimator suffers from two distinct error sources. First, finite-shot measurements introduce stochastic variance ($\sigma_q^2$), which scales inversely with the measurement shot budget. Second, physical imperfections in NISQ devices \citep{Preskill2018_boss}, such as decoherence and gate errors, systematically degrade the ideal unitary evolution into a noisy mixed state $\rho_{\mathrm{noisy}}(x)$. Thus, the parameter-shift rule evaluates a physically distorted objective landscape rather than the intended ideal function. Denoting the gradient estimator by $\widetilde g(x)$, its expectation is given by:
\begin{align*}
\mathbb{E}[\widetilde g(x)\mid x]
=
\nabla f(x) + b_q(x),
\end{align*}
where $b_q(x)$ is the hardware-induced gradient bias. We denote its bias scale by $U_q$, i.e., $\|b_q(x)\| \le U_q$.

Unlike finite-shot variance $\sigma_q^2$, this bias $U_q$ does not vanish by increasing the number of shots. 
In a federated setting, this creates an additional source of drift, that is, different clients may follow gradients distorted by their local quantum devices, causing their updates to move toward hardware-dependent noisy optima rather than the ideal global objective. 
Next, we introduce a technique, called Zero-Noise Extrapolation (ZNE), to mitigate this hardware bias.

\subsection{Zero-Noise Extrapolation (ZNE)}
\label{sec:zne_prelim}

ZNE is a quantum error-mitigation technique that reduces hardware-induced bias by estimating how circuit outputs change as the physical noise level is artificially increased~\citep{Temma2017_ZNE,Li2017_ZNE}. Let $\lambda=1$ denote the native hardware noise level. ZNE evaluates the same circuit at amplified noise levels
$
1=\lambda_1 < \lambda_2 < \cdots < \lambda_m
$. If $E(\lambda_k)$ denotes the noisy expectation value at noise level $\lambda_k$, then an extrapolation rule estimates the zero-noise value at $\lambda=0$. In linear extrapolation schemes, this estimator can be written as:
\begin{align*}
E_{\mathrm{ZNE}}
=
\sum_{k=1}^m \gamma_k E(\lambda_k),
\qquad
\sum_{k=1}^m \gamma_k = 1,
\end{align*}
where the weights $\{\gamma_k\}_{k=1}^m$ are chosen to cancel leading-order noise terms.

ZNE introduces a fundamental bias--variance trade-off. On one hand, extrapolation suppresses the systematic hardware bias. Specifically, if the raw quantum gradient oracle has bias scale $U_q$, we model the ZNE-corrected oracle as having residual bias $U_q/\kappa_b$ for some bias-suppression factor $\kappa_b\ge 1$. On the other hand, because the extrapolated estimate is a weighted combination of finite-shot measurements, its variance is amplified:
\begin{align*}
\mathrm{Var}(E_{\mathrm{ZNE}})
=
\sum_{k=1}^m \gamma_k^2 \mathrm{Var}(E(\lambda_k)).
\end{align*}
Since extrapolation weights often include large positive and negative coefficients, $\sum_k \gamma_k^2$ can be much larger than one~\citep{Takagi2022_ZNE}. Accordingly, if the raw quantum oracle has finite-shot variance $\sigma_q^2$, we model the ZNE-corrected oracle as having variance $\kappa_v\sigma_q^2$ for some variance-amplification factor $\kappa_v\ge 1$. Thus, ZNE reduces deterministic hardware bias at the cost of increased stochastic measurement variance.

\section{Related Works}
\label{sec:related_works}

\paragraph{Classical federated optimization.}
FedAvg~\citep{McMahan2017_originalFL} is a standard baseline for distributed training, that allows clients to perform multiple local SGD steps before server aggregation. However, under heterogeneous non-IID data, local objectives can deviate from the global objective, causing repeated local updates to drift toward client-specific optima. This phenomenon is known as \emph{client drift}~\citep{Zhao2018_related,Li2020_related}. Classical FL methods address this issue through several approaches. Specifically, FedProx~\citep{Li2020_related} constrains local updates with a proximal term, SCAFFOLD~\citep{Karimireddy_Scaffold} uses stateful control variates to correct local directions, and server-side momentum methods such as FedAvgM and FedOPT stabilize global updates through momentum or adaptive optimization~\citep{Hsu2019_FedAvgM,Reddi2021_related}. Nevertheless, these methods typically assume that client stochastic gradients are unbiased estimators of local gradients, i.e., $\mathbb{E}[g_i]=\nabla f_i$. This premise is violated in QFL, where gradients evaluated on noisy quantum hardware can contain systematic hardware-induced bias.

\paragraph{Quantum federated learning.}
QFL combines federated optimization with parameterized quantum circuits (PQCs). Early QFL works established the feasibility of training quantum or hybrid quantum--classical models across decentralized clients using FedAvg-style aggregation~\citep{Chen2021_QFLrelated,Xia2021_QFLrelated}. Motivated by privacy-preserving decentralized learning and the representational power of quantum models, QFL has since been explored in sensitive domains such as drug discovery~\citep{SAGGI2024_QFL_Drug,MESSINIS2025_QFL_Drug}, smart grids~\citep{Ren2024_QFLrelated}, intelligent transportation~\citep{Yamany2023_QFLrelated}, financial fraud detection~\citep{Innan2025_QFLrelated}, and medical image classification~\citep{Lusnig2024_QFLrelated}.

Recent work has begun to move beyond basic parameter averaging. Qi et al.~\citep{Qi2023_QFLrelated} improved local optimization through Federated Quantum Natural Gradient Descent. Zhao~\citep{Zhao2023_QFLrelated} addressed non-IID data using a one-shot QFL construction based on quantum-channel decomposition. Han et al.~\citep{Han2026_relatedQFL} proposed layerwise aggregation for clients with heterogeneous quantum circuit depths. These studies do not fully resolve the coupled optimization challenges of iterative noisy QFL. In particular, Zhao's one-shot construction does not model accumulated client drift over multiple communication rounds, while Han et al.'s layerwise method does not explicitly correct hardware-induced bias in quantum gradient estimation. Further details of related works can be found in Appendix~\ref{appendix:related_works}.


\section{Quantum Federated Optimization under Double Drift}
\label{sec:tech}
In this section, we develop the theoretical framework for QFL under double drift. We first present the problem setup and assumptions for quantum federated optimization. We then analyze the convergence of FedAvg, a commonly used aggregation method in QFL, and show how its performance is affected by the coupled effects of client drift and quantum hardware-induced gradient errors. Finally, we introduce \NAMEA, a ZNE-anchored control-variate method designed to mitigate this double-drift effect, and provide its convergence analysis.

\subsection{Problem Setup and Assumptions}
\label{sec:setup_assumptions}

We consider federated optimization of
\begin{align}
f(x) := \frac{1}{N}\sum_{i=1}^N f_i(x),
\end{align}
where $N$ is the number of clients, $f_i$ is the local objective of client $i$, and $x\in\mathbb{R}^d$ denotes the model parameters. At communication round $r$, the server samples a subset $\mathcal{S}_r$ of $S$ clients uniformly without replacement. Each selected client performs $K$ local steps with local stepsize $\eta_\ell$, and the server applies global stepsize $\eta_g$. We define the effective stepsize
\begin{align}
\widetilde{\eta}:=\eta_g\eta_\ell K,
\end{align}
and let $\mathcal{F}_{r-1}$ denote the filtration up to the round $r$.
In our setting, we use assumptions as follows:

\begin{assumption}[Smoothness]
\label{ass:smooth}
Each $f_i$ is $\beta$-smooth, i.e.,
\begin{align}
\|\nabla f_i(x)-\nabla f_i(y)\| \le \beta\|x-y\|,
\qquad \forall x,y\in\mathbb{R}^d .
\end{align}
\end{assumption}

\begin{assumption}[Client heterogeneity]
\label{ass:dissim}
There exist constants $G\ge 0$ and $B\ge 1$ such that, for all $x$,
\begin{align}
\frac{1}{N}\sum_{i=1}^N \|\nabla f_i(x)\|^2
\le
G^2 + B^2\|\nabla f(x)\|^2 .
\end{align}
\end{assumption}

\begin{assumption}[Classical stochastic gradients]
\label{ass:data}
For each client $i$, the noise-free stochastic gradient $g_i(x;\xi)$ is unbiased and has bounded variance:
\begin{align}
\mathbb{E}_{\xi}[g_i(x;\xi)\mid x] = \nabla f_i(x),
\qquad
\mathbb{E}_{\xi}\!\left[\|g_i(x;\xi)-\nabla f_i(x)\|^2\mid x\right]\le \sigma^2 .
\end{align}
\end{assumption}

\begin{assumption}[Quantum gradient oracle]
\label{ass:quantum}
For each client $i$, the hardware-executed quantum gradient estimator $\widetilde g_i(x;\xi)$ satisfies
\begin{align}
\mathbb{E}_q[\widetilde g_i(x;\xi)\mid x,\xi]
=
g_i(x;\xi)+b_{q,i}(x),
\qquad
\|b_{q,i}(x)\|\le U_q,
\end{align}
and its conditional quantum variance is bounded by
\begin{align}
\mathbb{E}_q\!\left[
\left\|
\widetilde g_i(x;\xi)
-
\mathbb{E}_q[\widetilde g_i(x;\xi)\mid x,\xi]
\right\|^2
\,\middle|\,x,\xi
\right]
\le \sigma_q^2 .
\end{align}
\end{assumption}

\begin{assumption}[Smooth quantum bias]
\label{ass:lipschitz_bias}
For each client $i$, the hardware bias map is $L_q$-Lipschitz:
\begin{align}
\|b_{q,i}(x)-b_{q,i}(y)\|
\le
L_q\|x-y\|,
\qquad \forall x,y\in\mathbb{R}^d .
\end{align}
\end{assumption}

Assumption~\ref{ass:smooth} is standard in nonconvex FL analysis and enables the usual smoothness-based descent argument~\citep{McMahan2017_originalFL,Karimireddy_Scaffold,Li2020_related}. Assumption~\ref{ass:dissim} quantifies non-IID client heterogeneity by controlling the gap between local and global gradient behavior, which is the main source of classical client drift in FedAvg-style methods~\citep{Li2020_related,Karimireddy_Scaffold}. Assumption~\ref{ass:data} is the standard unbiased mini-batch stochastic-gradient model used in FL, where stochasticity contributes variance but not systematic bias~\citep{Reddi2021_related}.

Assumption~\ref{ass:quantum} separates the two quantum errors considered in this work: finite-shot variance $\sigma_q^2$ and hardware-induced bias $U_q$. The former arises from estimating PQC expectation values with finitely many measurements and decreases with the shot budget~\citep{Sweke2020_variancequantum,Mari_variancequantum,Wierichs2022_PSR}, while the latter captures systematic distortion from NISQ hardware noise such as decoherence, gate errors, and readout errors~\citep{Preskill2018_boss,Wang2021_hardwarenoise,StilckFranca2021_hardwarenoise}. Assumption~\ref{ass:lipschitz_bias} states that this hardware bias changes smoothly with the PQC parameters, which is consistent with the smooth Fourier structure of PQC expectation landscapes~\citep{Schuld2021_smoothness}. Detailed justification for the quantum-specific assumptions is provided in Appendix~\ref{appendix:classical_fl_prelim}--\ref{appendix:quantum_noise_prelim}.

\subsection{Convergence of FedAvg in Quantum Federated Learning}
\label{sec:fedavg_analysis}

We first study how vanilla FedAvg behaves in QFL when local gradients are estimated by noisy parameterized quantum circuits. This analysis serves as a baseline for understanding how the classical FedAvg paradigm is affected by non-IID client data, finite-shot stochasticity, and hardware-induced quantum bias. At round $r$, each selected client initializes $y_{i,0}^r=x^{r-1}$ and performs $K$ local updates
\begin{align}
y_{i,k+1}^r
=
y_{i,k}^r-\eta_\ell \widetilde g_i(y_{i,k}^r;\xi_{i,k}^r),
\qquad k=0,\ldots,K-1 .
\end{align}
The server aggregates the local model changes as
\begin{align}
x^r
=
x^{r-1}
+
\eta_g \frac{1}{S}\sum_{i\in\mathcal S_r}(y_{i,K}^r-x^{r-1})
=
x^{r-1}-\widetilde\eta \widehat g^r,
\end{align}
where $\widetilde\eta=\eta_g\eta_\ell K$ and $\widehat g^r
:=
\frac{1}{SK}
\sum_{i\in\mathcal S_r}
\sum_{k=0}^{K-1}
\widetilde g_i(y_{i,k}^r;\xi_{i,k}^r).$

\begin{theorem}[FedAvg Convergence to Stationarity in QFL]
\label{thm:fedavg_final}
Suppose Assumptions~\ref{ass:smooth}--\ref{ass:quantum} hold. Let
$F:=f(x^0)-f^\star$, and let $\bar x^R$ be sampled uniformly from
$\{x^0,\ldots,x^{R-1}\}$. Assume
\begin{align}
\eta_\ell \le \frac{c_0}{(1+B^2)\beta K\eta_g},
\qquad
\eta_g\ge 1,
\end{align}
for a sufficiently small absolute constant $c_0>0$. Then FedAvg satisfies
\begin{align}
\mathbb E\|\nabla f(\bar x^R)\|^2
\le
\mathcal O\left(
\frac{F}{\widetilde\eta R}
+
\beta\widetilde\eta
\left(
\frac{G^2}{S}
+
\frac{\sigma^2+\sigma_q^2}{KS}
\right)
+
\beta^2\eta_\ell^2K^2
\left(
G^2+\sigma^2+\sigma_q^2
\right)
+
U_q^2
\right).
\end{align}
\end{theorem}

\paragraph{Proof sketch.}
The proof follows the standard smoothness-based descent argument, but must track the additional quantum bias. By $\beta$-smoothness,
\begin{align}
\mathbb E[f(x^r)\mid \mathcal F_{r-1}]
\le
f(x^{r-1})
+
\left\langle
\nabla f(x^{r-1}),
\mathbb E[\Delta x^r\mid\mathcal F_{r-1}]
\right\rangle
+
\frac{\beta}{2}
\mathbb E[\|\Delta x^r\|^2\mid\mathcal F_{r-1}],
\end{align}
where $\Delta x^r=x^r-x^{r-1}$. The first term is controlled by comparing the expected update direction with the true global gradient. Because local steps are evaluated at $y_{i,k}^r$ rather than the anchor point $x^{r-1}$, this comparison introduces a client-drift error. Because the quantum oracle satisfies
$
\mathbb E_q[\widetilde g_i(x;\xi)\mid x,\xi]
=
g_i(x;\xi)+b_{q,i}(x),
$
it also introduces a hardware-bias error of order $U_q^2$.

To bound the quadratic smoothness penalty, we decompose the averaged update direction into four components: the sampled population gradient, the local-drift deviation, the hardware-bias term, and the zero-mean stochastic noise. These yield contributions from client heterogeneity $G^2/S$, stochastic gradient variance $(\sigma^2+\sigma_q^2)/(KS)$, local drift, and hardware bias $U_q^2$. Finally, the local-drift energy is bounded by unrolling the $K$ local steps and using smoothness:
\begin{align}
E_r
\le
\eta_\ell^2K^2
\left(
B^2\|\nabla f(x^{r-1})\|^2
+
G^2
+
U_q^2
+
\sigma^2+\sigma_q^2
\right).
\end{align}
Substituting this drift bound into the one-round descent inequality, and telescoping over $R$ rounds, we have the result. The full derivation is provided in Appendix~\ref{appendix:fedavg_proof}.

\paragraph{Impact of double drift.}
Theorem~\ref{thm:fedavg_final} shows that FedAvg in QFL is affected by two distinct sources of drift. The classical federated drift appears through the heterogeneity terms
$$
\beta\widetilde\eta\frac{G^2}{S}
\qquad\text{and}\qquad
\beta^2\eta_\ell^2K^2G^2,
$$
which arise because clients perform multiple local steps on non-IID data. They can be reduced by sampling more clients, using smaller local steps, or reducing the number of local updates.

In contrast, the quantum hardware drift appears through the residual term
$$
U_q^2 .
$$
This term is qualitatively different from stochastic variance. The finite-shot variance $\sigma_q^2$ is averaged over local steps and participating clients through the factor $(KS)^{-1}$, but the hardware-induced bias $U_q$ persists even after averaging. 
Thus, this creates an error floor for vanilla FedAvg.

\subsection{\NAMEA: Quantum Federated ZNE-Anchored Control}
\label{sec:qfedzac}

To address the double-drift issue, we propose \NAMEA, a novel QFL aggregation framework that combines stateful control variates with ZNE correction. The key design is asymmetric. Specifically, clients perform inexpensive local updates using the standard noisy quantum oracle, while the server maintains a more accurate global reference control using ZNE-corrected anchor gradients. The full procedure is summarized in Algorithm~\ref{alg:qfedzac}.

At round $r$, each selected client $i\in\mathcal S_r$ performs corrected local updates
\begin{align}
y_{i,k+1}^r
=
y_{i,k}^r
-
\eta_\ell
\Bigl(
\widetilde g_i(y_{i,k}^r;\xi_{i,k}^r)
-
c_i^{r-1}
+
c_{\mathrm{srv}}^{r-1}
\Bigr),
\qquad
k=0,\dots,K-1,
\label{eq:qfedzac_local}
\end{align}
where $c_i^{r-1}$ is the client control and $c_{\mathrm{srv}}^{r-1}$ is the server reference control. The server aggregates the resulting local model changes as
\begin{align}
x^r
=
x^{r-1}
+
\eta_g\frac{1}{S}
\sum_{i\in\mathcal S_r}(y_{i,K}^r-x^{r-1}).
\label{eq:qfedzac_server_update}
\end{align}

The client control is updated with the standard noisy oracle, while the server-side reference is updated through a ZNE-corrected control memory:
\begin{align}
c_i^r
&=
(1-\alpha)c_i^{r-1}
+
\alpha \widetilde g_i(x^{r-1};\xi_i^r),
\label{eq:qfedzac_client_control}
\\
c_{i,\mathrm{ZNE}}^r
&=
(1-\alpha)c_{i,\mathrm{ZNE}}^{r-1}
+
\alpha \widetilde g_i^{\mathrm{ZNE}}(x^{r-1};\bar \zeta_i^r),
\label{eq:qfedzac_zne_control}
\\
c_{\mathrm{srv}}^r
&=
c_{\mathrm{srv}}^{r-1}
+
\frac{1}{N}\sum_{i\in\mathcal S_r}
\left(
c_{i,\mathrm{ZNE}}^r-c_{i,\mathrm{ZNE}}^{r-1}
\right).
\label{eq:qfedzac_server_control}
\end{align}

Intuitively, \NAMEA\ mitigates double-drift by substituting the client's local update trajectory ($c_i^{r-1}$) with a ZNE-anchored global reference ($c_{\mathrm{srv}}^{r-1}$). This realignment simultaneously mitigates the client drift and steers the network away from hardware-induced false optima. Although ZNE inherently inflates quantum measurement variance ($\kappa_v \sigma_q^2$) (Sections~\ref{sec:zne_prelim}, \ref{appendix:zne_prelim}), the momentum parameter $\alpha$ acts as a critical low-pass filter to safely absorb this explosive statistical noise. Thus, \NAMEA\ successfully co-designs stateful momentum and error mitigation to resolve both classical and quantum drift. The next subsection formalizes this intuition.

\begin{algorithm}[t]
\caption{\NAMEA: Quantum Federated ZNE-Anchored Control}
\label{alg:qfedzac}
\begin{algorithmic}[1]
\Require Initial model $x^0$, local steps $K$, stepsizes $\eta_\ell,\eta_g$, momentum $\alpha$, client sample size $S$
\State Initialize $c_i^0=c_{i,\mathrm{ZNE}}^0=0$ for all $i\in[N]$, and $c_{\mathrm{srv}}^0=0$
\For{$r=1,\dots,R$}
    \State Sample clients $\mathcal S_r\subseteq[N]$ with $|\mathcal S_r|=S$
    \For{each $i\in\mathcal S_r$ in parallel}
        \State $y_{i,0}^r \gets x^{r-1}$
        \For{$k=0,\dots,K-1$}
            \State $y_{i,k+1}^r \gets y_{i,k}^r-\eta_\ell\bigl(\widetilde g_i(y_{i,k}^r;\xi_{i,k}^r)-c_i^{r-1}+c_{\mathrm{srv}}^{r-1}\bigr)$
        \EndFor
        \State $c_i^r \gets (1-\alpha)c_i^{r-1}+\alpha\widetilde g_i(x^{r-1};\xi_i^r)$
        \State $c_{i,\mathrm{ZNE}}^r \gets (1-\alpha)c_{i,\mathrm{ZNE}}^{r-1}+\alpha\widetilde g_i^{\mathrm{ZNE}}(x^{r-1};\bar\zeta_i^r)$
        \State Send $y_{i,K}^r$, $\Delta c_i^r$, and $\Delta c_{i,\mathrm{ZNE}}^r$ to server
    \EndFor
    \State For $i\notin\mathcal S_r$, set $c_i^r=c_i^{r-1}$ and $c_{i,\mathrm{ZNE}}^r=c_{i,\mathrm{ZNE}}^{r-1}$
    \State $x^r \gets x^{r-1}+\eta_g\frac{1}{S}\sum_{i\in\mathcal S_r}(y_{i,K}^r-x^{r-1})$
    \State $c_{\mathrm{srv}}^r \gets c_{\mathrm{srv}}^{r-1}+\frac{1}{N}\sum_{i\in\mathcal S_r}\Delta c_{i,\mathrm{ZNE}}^r$
\EndFor
\State \Return $\bar x^R$ sampled uniformly from $\{x^0,\dots,x^{R-1}\}$
\end{algorithmic}
\end{algorithm}

\paragraph{Convergence analysis.}
To analyze \NAMEA, define the client-control errors
\begin{align}
\Gamma_r
:=
\frac{1}{N}\sum_{i=1}^N
\mathbb E\|c_i^r-\mu_i(x^r)\|^2,
\qquad
\mu_i(x):=\nabla f_i(x)+b_{q,i}(x),
\end{align}
and
\begin{align}
\Gamma_r^{\mathrm{ZNE}}
:=
\frac{1}{N}\sum_{i=1}^N
\mathbb E\|c_{i,\mathrm{ZNE}}^r-\mu_{i,\mathrm{ZNE}}(x^r)\|^2,
\qquad
\mu_{i,\mathrm{ZNE}}(x):=\nabla f_i(x)+b_{q,i}^{\mathrm{ZNE}}(x).
\end{align}
We use the Lyapunov function
\begin{align}
\Phi_r
:=
\mathbb E[f(x^r)]
+
\lambda \Gamma_r
+
\lambda_{\mathrm{ZNE}}\Gamma_r^{\mathrm{ZNE}},
\qquad
\lambda,\lambda_{\mathrm{ZNE}}
=
C_\lambda \widetilde\eta\frac{N}{\alpha S}.
\end{align}

\begin{theorem}[\NAMEA\ Convergence to Stationarity]
\label{thm:qfedzac_final}
Suppose Assumptions~\ref{ass:smooth}--\ref{ass:lipschitz_bias} hold and
\begin{align}
\widetilde\eta
\le
\frac{c_1}{\sqrt{\beta^2+L_q^2}}\frac{S}{N},
\qquad
\eta_g\ge 1.
\end{align}
Let $\alpha=\mathcal O(\widetilde\eta N/S)$, $F:=\Phi_0-f^\star$, and let $\bar x^R$ be sampled uniformly from $\{x^0,\dots,x^{R-1}\}$. Then \NAMEA\ satisfies
\begin{align}
\mathbb E\|\nabla f(\bar x^R)\|^2
\le
\mathcal O\left(\frac{F}{\widetilde\eta R}\right)
+
\mathcal O\left(
\widetilde\eta\frac{N}{S}
(\sigma^2+\kappa_v\sigma_q^2)
\right)
+
\mathcal O\left(
\frac{U_q^2}{\kappa_b^2}
+
\widetilde\eta U_q^2
\right).
\end{align}
\end{theorem}

\paragraph{Proof sketch.}
The proof follows a Lyapunov descent argument. First, the corrected local direction in~\eqref{eq:qfedzac_local} is compared with the global gradient. Because the local control $c_i^{r-1}$ tracks the same biased oracle used in the client update, the leading anchor-point hardware bias is canceled in the corrected direction. The remaining first-order errors are due to local drift, stale control estimation, and the residual ZNE bias $U_q/\kappa_b$ in the server reference.

Second, the server update is bounded by decomposing the corrected direction into the true global gradient, local drift, client-control lag, server-control mismatch, hardware bias, and zero-mean noise. This yields variance terms from the unmitigated client oracle and tracking-error terms $\Gamma_r$ and $\Gamma_r^{\mathrm{ZNE}}$.

Third, the control memories are shown to contract because each sampled client updates its control variate by an exponential moving average. Since only a fraction $S/N$ of clients participates per round, the contraction rate scales as $\alpha S/N$. The ZNE control obeys the same recursion, but with variance $\sigma^2+\kappa_v\sigma_q^2$ due to ZNE variance amplification.

Finally, by choosing
$
\lambda,\lambda_{\mathrm{ZNE}}
=
C_\lambda \widetilde\eta\frac{N}{\alpha S}$, we allow the Lyapunov decrease in the control errors to absorb the first-order tracking errors appearing in the objective descent. Then, we have:
\begin{align}
\Phi_r
\le
\Phi_{r-1}
-
c\widetilde\eta\,\mathbb E\|\nabla f(x^{r-1})\|^2
+
C\widetilde\eta
\left(
\widetilde\eta U_q^2
+
\frac{U_q^2}{\kappa_b^2}
+
\alpha(\sigma^2+\kappa_v\sigma_q^2)
\right).
\end{align}
Telescoping this inequality over $R$ rounds and using $\alpha=\mathcal O(\widetilde\eta N/S)$ yields Theorem~\ref{thm:qfedzac_final}. The full derivation is deferred to Appendix~\ref{appendix:qfedzac_proof}.

\paragraph{Discussion.}
Theorem~\ref{thm:qfedzac_final} shows how \NAMEA\ mitigates the double-drift effect. Compared with FedAvg, the explicit heterogeneity terms involving $G^2$ are removed from the final stationarity bound because the stateful control variates correct accumulated client drift. Meanwhile, the raw quantum bias floor $U_q^2$ is replaced by two smaller residuals: a ZNE-controlled bias term $U_q^2/\kappa_b^2$ and a step-size-weighted term $\widetilde\eta U_q^2$. Thus, as $\widetilde\eta$ decreases and ZNE becomes more accurate, the hardware-bias error is substantially reduced.

The fundamental cost of this bias reduction is variance amplification. Specifically, ZNE explicitly inflates the quantum measurement variance from $\sigma_q^2$ to $\kappa_v\sigma_q^2$. If injected directly into the global model, this amplified noise would break convergence. However, in \NAMEA, this penalty is strictly filtered through the momentum parameter $\alpha$. Because $\alpha$ (scaled by $\widetilde\eta N/S$) acts as a mathematical dampener, it effectively absorbs the ZNE-induced volatility. As a result, the algorithm successfully trades a deterministic hardware bias for a rigorously bounded variance penalty.
\begin{figure}[t!]
        \centering
        \includegraphics[width=.35\textwidth]{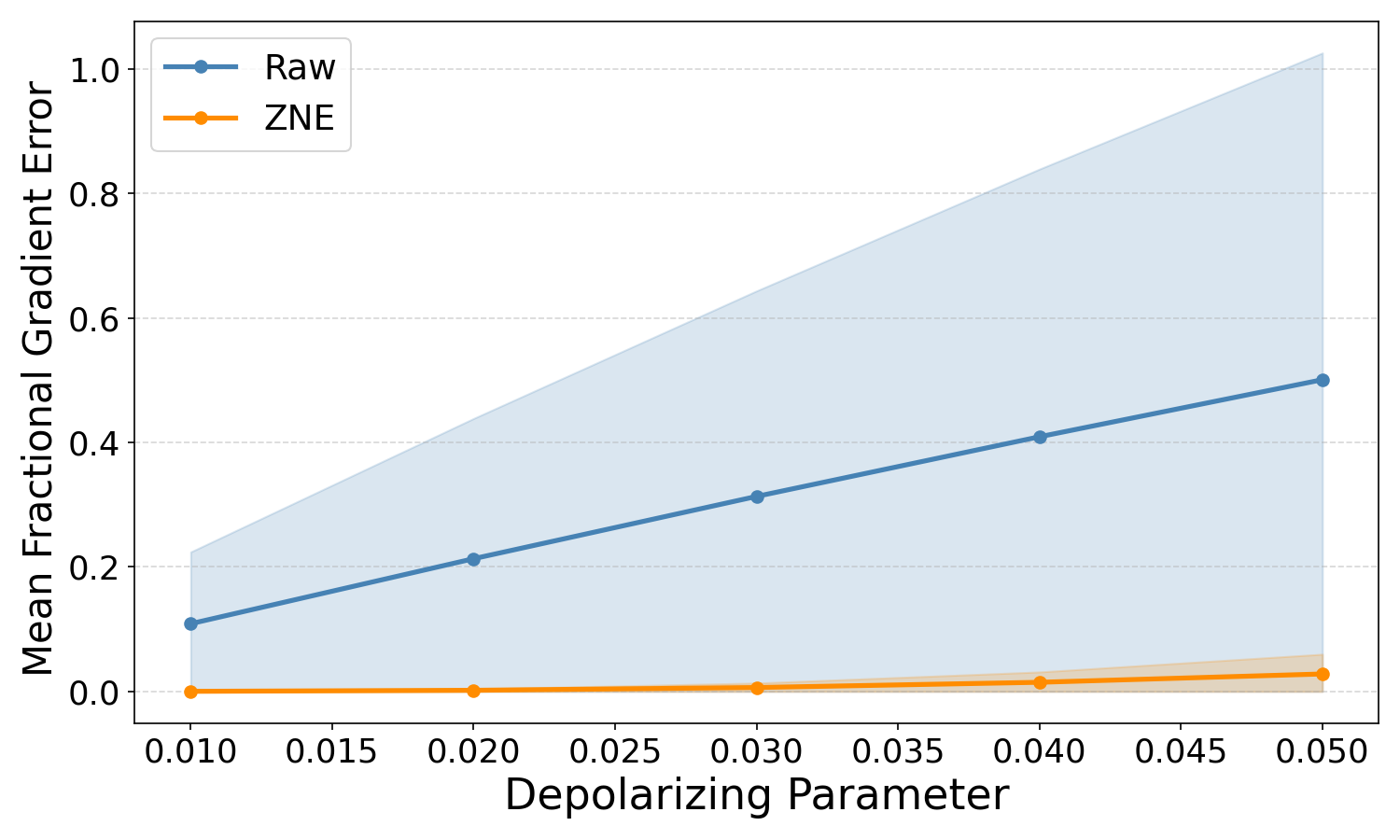}
        \caption{Impact of hardware noise on gradient bias.}
        \label{fig:bias_ref_raw_vs_zne}
\end{figure}
\section{Experimental Results}
\label{sec:exp}
In this section, we first perform a sensitivity analysis to assess the impact of underlying hardware noise on the bias of QFL gradient computation. 
Then, we benchmark \NAMEA\ against FedAvg and SCAFFOLD under the QFL setting, including non-IID distributions and hardware-induced bias from noisy quantum gradient estimates. We show that under this setting, \NAMEA\ achieves more stable training and higher test performance compared to FedAvg and SCAFFOLD.

\paragraph{Implementation Details.} We evaluate our approach on the Binary Blobs~\citep{bowles2025binaryblobs} dataset, a benchmark dataset encouraged by up-to-date QML works~\citep{bowles2024better,Ngo2026_qshiftdp,recioarmengol2026_Blob}, rather than traditional datasets. This choice aligns with recent QML literature demonstrating that the aggressive dimensionality reduction required to fit regular images (e.g., MNIST images) onto near-term quantum devices obscures meaningful performance evaluation \citep{bowles2024better}. 
The dataset includes $8$ classes. We use a split of $5,000$ training and $10,000$ testing samples. We implement a QML model designed to incorporate \NAMEA. The selected QML architecture is widely adopted in prior QML studies (\citep{Havlicek2019_QML,du2021quantum, watkins2023quantum,long2025hybrid}). The model is optimized via Negative Log-Likelihood (NLL) loss.
 To simulate noisy NISQ hardware, depolarizing noise with a strength $p$ is applied directly as a channel in the QML model. For ZNE, we use scale factors $\lambda \in \{1.0, 3.0, 5.0\}$, and the gradient is estimated via second-order polynomial extrapolation.

\paragraph{Federated Learning Settings.} We simulate a federated network of $8$ clients, partitioning the data via a Dirichlet distribution ($\alpha_\text{dirichlet} = 0.3$) to induce statistical non-IID heterogeneity. Training spans $20$ communication rounds with full client participation. Locally, clients perform $5$ epochs with a batch size of $16$ using Stochastic Gradient Descent (SGD) with a local learning rate of $\eta_l = 0.1$, a momentum factor of $0.9$. The server aggregates these updates using FedAvg or SCAFFOLD with a global learning rate of $\eta_g = 1.0$. For \NAMEA, we set anchor momentum $\alpha=0.1$, which corresponds to our theoretical analysis that $\alpha=\mathcal O(\widetilde\eta N/S)$. Details are given in Appx. \ref{appendix:detailsdatasets}.

\subsection{Sensitivity analysis of hardware-induced bias}

To evaluate hardware-induced gradient bias, we compare noisy gradient evaluations against the exact analytical gradient under increasing depolarizing strength ($p=0.01$ to $0.05$). We define the fractional gradient error as
$\frac{\|\widetilde g - g_{\mathrm{ideal}}\|}{\|g_{\mathrm{ideal}}\|},$
where $g_{\mathrm{ideal}}$ is the exact noiseless gradient and $\widetilde g$ is the gradient evaluated under hardware noise. As shown in Fig.~\ref{fig:bias_ref_raw_vs_zne}, the raw fractional gradient error grows substantially with the noise level, indicating that hardware noise can induce a non-negligible bias relative to the ideal gradient. In contrast, applying ZNE consistently suppresses this bias across all tested noise levels, supporting our model that ZNE reduces the hardware-bias scale from $U_q$ to a smaller residual $U_q/\kappa_b$.


\subsection{Performance of \NAMEA\ under quantum noise}
We compare the performance of \NAMEA, FedAvg and SCAFFOLD under double-drift phenomenon which includes client drift from non-IID dataset and hardware drift from quantum noise. Fig. \ref{fig:depo001} shows that \NAMEA\ achieves the best training and testing performance. This result is consistent with the theory: under hardware-induced bias, FedAvg suffers from a persistent error floor proportional to  $U_q^2$ , while SCAFFOLD can partially mitigate client drift but does not explicitly correct the bias introduced by hardware noise. We also show that this trend persists under larger depolarizing noise in Appx. \ref{appx:moreperf}. We measure the impact of varying finite-shot variance $\sigma_q^2$ in Appx. \ref{appendix:shotnoise}.
\begin{figure}[h!]
    \centering
    \includegraphics[width=\linewidth]{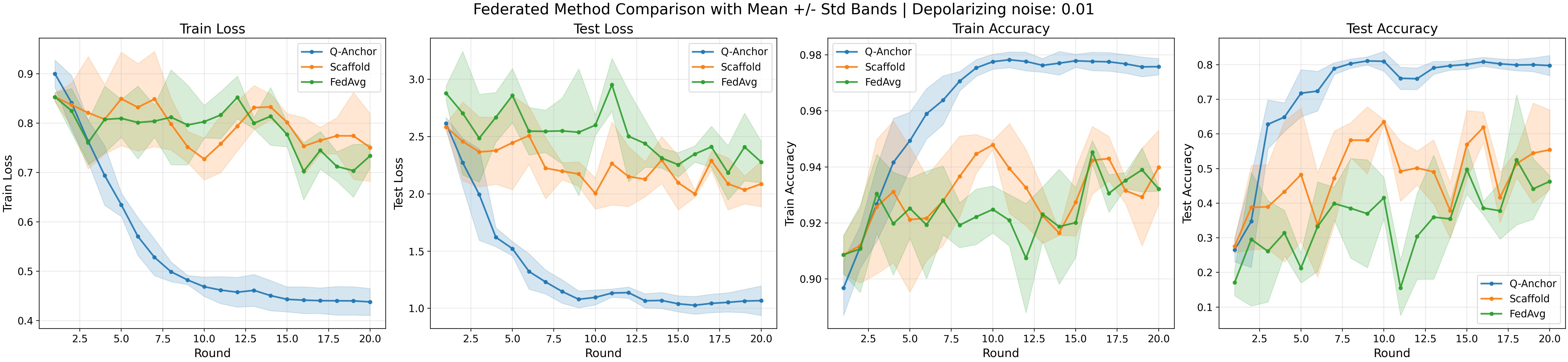}
    \caption{Performance comparison of \NAMEA\ and other FL baselines under depolarizing noise $p=0.01$. }
    \label{fig:depo001}
\end{figure}

\section{Conclusion}
\label{sec:con}
In this work, we studied quantum federated learning under the double-drift phenomenon including client heterogeneity and noisy quantum gradient estimation. We first showed that vanilla FedAvg suffers from a hardware-bias error floor in realistic QFL settings. To address this issue, we proposed \NAMEA, a ZNE-anchored control-variate framework that jointly mitigates classical client drift and quantum hardware-induced bias. Our theoretical analysis shows that \NAMEA\ reduces the bias floor while controlling the variance amplification introduced by ZNE, and our experiments demonstrate more stable performance compared with FedAvg and SCAFFOLD. One limitation of this work is that our empirical evaluation is conducted in simulated noisy quantum environments. Future work will examine \NAMEA\ in real quantum hardware.

\bibliography{references}
\bibliographystyle{unsrt}


\newpage
\appendix
\setcounter{assumption}{0}
\setcounter{theorem}{0}
\section*{Appendix}

\section{Extended Preliminaries}
\label{appendix:prelim}
This section establishes the mathematical and physical foundations of Quantum Federated Learning. We first formalize the standard classical federated optimization framework. We then transition to the quantum domain by introducing the core mechanics of quantum computation and detailing how the physical evaluation of quantum gradients fundamentally diverges from classical optimization paradigms.

\subsection{Classical Federated Optimization and Client Drift}
\label{appendix:classical_fl_prelim}

In classical Federated Learning (FL), the goal is to collaboratively train a global model $x \in \mathbb{R}^d$ across a large federation of $N$ clients without centralizing their localized datasets. Formally, this is cast as the empirical risk minimization problem:
\begin{align}
\min_{x \in \mathbb{R}^d} f(x) := \frac{1}{N} \sum_{i=1}^N f_i(x), \qquad \text{where} \quad f_i(x) := \mathbb{E}_{\xi_i \sim \mathcal{D}_i}[\ell(x; \xi_i)].
\end{align}
Here, $f_i(x)$ is the local objective function for client $i$, defined as the expected loss $\ell$ over its local data distribution $\mathcal{D}_i$. To minimize communication overhead, algorithms like Federated Averaging (FedAvg) \citep{McMahan2017_originalFL} require a sampled subset of $S$ clients to perform multiple local gradient descent steps before the server aggregates their updated models. 

The theoretical convergence of these algorithms is governed by two fundamental classical error sources:

\paragraph{Classical Stochastic Variance.}
During local training, clients cannot evaluate the full expectation $\nabla f_i(x)$. Instead, they compute stochastic mini-batch gradients. Let $g_{c,i}(x; \xi_i)$ denote the classical stochastic gradient evaluated on a mini-batch $\xi_i$. This estimator introduces statistical noise. However, because the data is drawn directly from the client's true local distribution, this noise is strictly zero-mean:
\begin{align}
\mathbb E[g_{i}(x; \xi_i)] = \nabla f_i(x).
\end{align}
We assume the variance of this classical estimator is bounded by a constant $\sigma^2$ (Assumption~\ref{ass:data}). Because this noise is unbiased, stochastic gradient descent seamlessly averages it out over time.

\paragraph{Classical Client Drift (Data Heterogeneity).}
The primary challenge in classical FL arises from non-IID data distributions across clients ($\mathcal{D}_i \neq \mathcal{D}_j$). As a result, a local gradient is a biased estimator of the \emph{global} gradient ($\nabla f_i(x) \neq \nabla f(x)$). In our framework, we bound this spatial heterogeneity by $G^2$ (Assumption~\ref{ass:dissim}). When clients perform multiple local steps using heterogeneous data, their models drift toward local optima rather than the global minimum. This phenomenon is known as \emph{client drift}.

\paragraph{The Paradigm Shift: Classical vs. Quantum FL.}
While advanced classical FL algorithms can successfully correct the spatial bias caused by data heterogeneity, they rely on one foundational premise, that is, the local gradient estimators themselves must be unbiased with respect to the local objective ($\mathbb E[g_{i}] = \nabla f_i$). 

In quantum FL, transitioning the classical learning models to Parameterized Quantum Circuits (PQCs) fundamentally shatters this premise. As detailed in the following section, evaluating gradients on physical quantum hardware introduces errors that behave entirely differently than classical mini-batch noise. While quantum measurement shot noise mimics classical variance ($\sigma_q^2$), physical device decoherence injects a systematic, deterministic hardware bias into the gradient evaluation. Unlike classical data heterogeneity—which is an artifact of the network distribution—this quantum bias is an artifact of the physics itself. It forces $\mathbb E[g_{i}] \neq \nabla f_i$, creating a \emph{double drifting} problem that standard FL algorithms are mathematically unequipped to solve.

\subsection{Quantum Federated Learning and Hardware Drift}
\label{appendix:quantum_noise_prelim}

To contextualize the optimization challenges of Quantum Federated Learning, we first briefly review the fundamentals of quantum computation and the role of Parameterized Quantum Circuits (PQCs). After establishing how exact quantum gradients are derived via the parameter-shift rule, we detail how execution on physical hardware corrupts these analytical gradients. Specifically, we delineate the two fundamental forms of quantum error: statistical variance arising from finite measurements, and systemic hardware bias caused by device decoherence.

\paragraph{Quantum Computing Basics.}
In quantum information, the state of an $n$-qubit system is described by a state vector $|\psi\rangle$ residing in a complex Hilbert space $\mathcal{H}$. A pure quantum state is expressed as a superposition of computational basis states $\{|z\rangle\}_{z \in \{0,1\}^n}$:
\begin{align*}
|\psi\rangle = \sum_{z \in \{0,1\}^n} \alpha_z |z\rangle, \qquad \text{where} \quad \sum_{z} |\alpha_z|^2 = 1,
\end{align*}
and $\alpha_z \in \mathbb{C}$ are the probability amplitudes. Closed-system quantum evolution is governed by unitary operators $U$ (such that $U^\dagger U = I$), which map an initial state $|\psi\rangle$ to a new state $U|\psi\rangle$. 

To extract classical information, one measures the quantum state with respect to a physical observable, represented by a Hermitian matrix $O = O^\dagger$. The theoretical expectation value of this observable for a pure state $|\psi\rangle$ is given by:
\begin{align*}
\langle O \rangle = \langle \psi | O | \psi \rangle.
\end{align*}
More generally, an open or noisy quantum system cannot be described by a single pure state vector. It exists as a statistical ensemble of pure states, mathematically represented by a positive semi-definite density matrix $\rho$ with unit trace ($\mathrm{Tr}(\rho) = 1$). For a general mixed state $\rho$, the expected observable value generalizes to $\langle O \rangle = \mathrm{Tr}(O \rho)$.

\paragraph{Parameterized Quantum Circuits (PQCs).}
In quantum machine learning, the core computational workload is delegated to a Parameterized Quantum Circuit (PQC). Let $x \in \mathbb{R}^d$ denote the classical trainable parameters. A PQC applies a parameterized unitary sequence $U(x)$ to a standardized reference state (typically $|0\rangle^{\otimes n}$) to prepare a parameterized output state $|\psi(x)\rangle = U(x)|0\rangle^{\otimes n}$. The objective function $f(x)$ is defined as the expected value of a problem-specific observable $O$:
\begin{align*}
f(x) = \langle \psi(x) | O | \psi(x) \rangle = \mathrm{Tr}\Bigl(O |\psi(x)\rangle\langle\psi(x)|\Bigr).
\end{align*}
To optimize this landscape, the classical-quantum hybrid system must evaluate the gradient $\nabla f(x)$. Specifically, exact analytical derivatives can be evaluated directly on the quantum hardware using the \emph{parameter-shift rule}, which computes the gradient by evaluating the objective at shifted parameter configurations.

\paragraph{The Parameter-Shift Rule.}
Unlike classical neural networks, quantum circuits cannot utilize standard backpropagation because intermediate quantum states cannot be stored or copied due to the no-cloning theorem. Instead, exact analytical derivatives are evaluated directly on the quantum hardware using the \emph{parameter-shift rule} \citep{Mitarai2018_PSR, Schuld2019_PSR,Wierichs2022_PSR}. For a PQC constructed using standard Pauli rotation gates (e.g., $U(x_j) = \exp(-i \frac{x_j}{2} P)$), the partial derivative with respect to the $j$-th parameter $x_j$ is given exactly by a macroscopic finite difference:
\begin{align*}
\frac{\partial f(x)}{\partial x_j} = \frac{1}{2} \left[ f\left(x + \frac{\pi}{2} e_j\right) - f\left(x - \frac{\pi}{2} e_j\right) \right],
\end{align*}
where $e_j$ is the $j$-th standard basis vector. By applying this rule dimension-wise, one constructs the full analytical gradient $\nabla f(x)$. However, because this calculation relies on evaluating physical expectation values, executing it on near-term hardware subjects the resulting gradient estimator to two fundamental forms of corruption: measurement noise and hardware noise.

\paragraph{Measurement Noise (Statistical Variance).}
The first source of error comes from the physics of quantum measurement. In theory, our objective function relies on an exact, continuous expectation value $\langle O \rangle$. However, we cannot access this continuous number directly. When we physically measure a quantum state, the state collapses. This collapse forces the computer to output a single, discrete value (an eigenvalue of $O$) based on quantum probabilities. 

To estimate the true continuous average, we must run the identical circuit many times. We allocate a budget of $B_c$ independent measurement shots and average the discrete results. This finite sampling injects zero-mean statistical noise into our gradient calculation.

On the other hand, this variance is mathematically bounded. The variance of a quantum measurement cannot exceed the squared maximum eigenvalue of the observable, $\|O\|_\infty^2$. Foundational research in stochastic quantum optimization \citep{Sweke2020_variancequantum, Mari_variancequantum,Wierichs2022_PSR,Ngo2026_qshiftdp} proves that the variance of the resulting empirical gradient $\widehat{\nabla f(x)}$ scales strictly inversely with the shot budget:
\begin{align}
\mathrm{Var}(\widehat{\nabla f(x)}) \le \mathcal{O}\!\left(\frac{\|O\|_\infty^2}{B_c}\right).
\end{align}
This bounded physical property justifies our analytical framework (Assumption~\ref{ass:quantum}). It confirms the quantum variance is bounded to a constant $\sigma_q^2 = \mathcal{O}(1/B_c)$. Because this statistical noise is zero-mean, it contributes only to the variance of the quantum gradients, not to their bias.

\paragraph{Hardware Noise (Systemic Bias).}
The second source of error stems from the physical realities of Noisy Intermediate-Scale Quantum (NISQ) devices \citep{Preskill2018_boss}. Physical qubits continuously interact with their environment, leading to decoherence phenomena, such as thermal relaxation ($T_1$) and dephasing ($T_2$), alongside coherent gate infidelities, crosstalk, and measurement readout errors.

Unlike finite-shot measurement noise, hardware noise fundamentally alters the underlying quantum system. Instead of acting as ideal unitaries, noisy quantum operations act as Completely Positive Trace-Preserving (CPTP) maps, or \emph{quantum channels}. As a result, PQCs do not prepare the ideal pure state $|\psi(x)\rangle\langle\psi(x)|$, but rather a degraded mixed state $\rho_{\mathrm{noisy}}(x)$. 

Recent work has shown that hardware noise can both suppress gradient magnitudes and deform the optimization landscape of variational quantum algorithms \citep{Wang2021_hardwarenoise, StilckFranca2021_hardwarenoise}. In particular, incoherent noise can induce noise-induced barren plateaus, while more general coherent and incoherent noise processes can alter the optimal parameters and optimal cost of the noisy variational problem. As a result, the parameter-shift rule effectively evaluates the gradient of an entirely different, physically distorted optimization landscape.

Let $\widetilde g(x;\xi)$ denote the empirical gradient estimator evaluated on this deformed physical hardware. Its expected value is systematically shifted away from the true analytical gradient by an input-dependent hardware bias $b_q(x)$:
\begin{align*}
\mathbb E_q[\widetilde g(x;\xi)\mid x,\xi] = \nabla f(x) + b_q(x).
\end{align*}

Crucially, both the magnitude and the smoothness of this hardware bias are physically and mathematically bounded. First, by the axioms of quantum mechanics~\citep{nielsen2010quantum}, any physical noise process must map to a valid density matrix. Thus, its observable expectation value can never exceed the spectral norm of the observable, $|\mathrm{Tr}(O \rho_{\mathrm{noisy}})| \le \|O\|_\infty$. Because the parameter-shift rule constructs the gradient via a finite linear combination of these bounded expectation values, the maximum possible deviation between the true analytical gradient and the noisy hardware gradient is mathematically restricted. This rigorously justifies Assumption~\ref{ass:quantum} ($\|b_q(x)\| \le U_q$).

Furthermore, this systemic bias is not a discontinuous function. Instead, it inherits strict smoothness properties from the quantum circuit architecture. Because PQCs are constructed using Pauli generators, the output expectation values of both the ideal and the physically noisy circuits manifest as finite Fourier series with respect to the trainable parameters \citep{Schuld2021_smoothness}. Consequently, both the true analytical landscape and the noise-deformed landscape possess bounded second derivatives. Because the hardware bias $b_q(x)$ is simply the difference between two smooth, globally Lipschitz gradients, the bias map itself is mathematically guaranteed to be $L_q$-Lipschitz continuous (Assumption~\ref{ass:lipschitz_bias}). 

While this bounded, it remains deterministic with respect to the measurement shot count. Specifically, evaluating $B_c \to \infty$ shots will not reduce $U_q$. Thus, this hardware bias establishes an artificial error floor. In a federated setting, this forces individual clients to converge toward hardware-specific false optima rather than the true global solution, severely exacerbating client drift and necessitating advanced mitigation strategies.

\subsection{Zero-Noise Extrapolation (ZNE)}
\label{appendix:zne_prelim}

In this part, we review Zero-Noise Extrapolation (ZNE), a quantum error-mitigation technique designed to systematically reduce hardware-induced bias \citep{Temma2017_ZNE, Li2017_ZNE}. This provides a concrete approach for mitigating the hardware bias $U_q$, while clarifying the trade-offs associated with its use.

\paragraph{The Mechanism of ZNE.}
The core mechanism of ZNE relies on intentionally making the quantum circuit \emph{worse} to understand how noise affects the output, and then mathematically projecting backwards to a noiseless state. Let $\lambda = 1$ represent the baseline physical noise level of the quantum hardware. ZNE scales this noise to higher levels $\lambda \in \{\lambda_1, \lambda_2, \dots, \lambda_m\}$ where $1 = \lambda_1 < \lambda_2 < \dots < \lambda_m$. In practice, this noise amplification is achieved either by stretching the microwave control pulses on the hardware \citep{Kandala2019_ZNE} or by digitally compiling identity insertions (e.g., replacing a unitary $U$ with $UU^\dagger U$) into the circuit \citep{Giurgica2020_ZNE}. 

At each amplified noise level $\lambda_k$, the quantum hardware evaluates the noisy expectation value, yielding a parameterized output $E(\lambda_k)$. Finally, a classical curve-fitting algorithm (such as Richardson, polynomial, or exponential extrapolation) is applied to these noisy points to infer the idealized, zero-noise expectation value at $\lambda = 0$. Mathematically, linear extrapolation methods formulate the ZNE-mitigated output $E_{\mathrm{ZNE}}$ as a weighted linear combination of the noisy evaluations:
\begin{align}
E_{\mathrm{ZNE}} = \sum_{k=1}^m \gamma_k E(\lambda_k), \qquad \text{where} \quad \sum_{k=1}^m \gamma_k = 1.
\end{align}

\paragraph{The Bias-Variance Trade-off.}
For the purposes of our stochastic optimization analysis, we capture the physical consequences of this extrapolation through two effective scalar parameters:

\begin{itemize}
    \item \textbf{Bias-suppression factor ($\kappa_b \ge 1$):} By carefully choosing the extrapolation weights $\gamma_k$, ZNE mathematically cancels the leading-order terms of the Taylor expansion of the noise channel. Consequently, the magnitude of the hardware-induced bias is severely reduced. If the raw quantum oracle has a systemic bias scale bounded by $U_q$, the ZNE-corrected oracle exhibits a residual bias scale of $U_q/\kappa_b$.
    \item \textbf{Variance-amplification factor ($\kappa_v \ge 1$):} Because $E_{\mathrm{ZNE}}$ is a linear combination of independent, finite-shot quantum measurements, its statistical variance is the weighted sum of the individual variances: $\mathrm{Var}(E_{\mathrm{ZNE}}) = \sum_k \gamma_k^2 \mathrm{Var}(E(\lambda_k))$. Because the extrapolation weights $\gamma_k$ typically take large positive and negative values to project backward to $\lambda=0$, the sum of their squares strictly exceeds $1$ ($\sum \gamma_k^2 \gg 1$) \citep{Takagi2022_ZNE}. Therefore, if the raw quantum oracle has a finite-shot variance of $\sigma_q^2$, the ZNE procedure inflates this uncertainty to $\kappa_v \sigma_q^2$.
\end{itemize}

This bias-variance trade-off is the fundamental physical cost of ZNE. Specifically, one suppresses the deterministic hardware error $U_q$ by aggressively inflating the stochastic measurement fluctuation $\sigma_q^2$.

\section{Extended Related Works}
\label{appendix:related_works}

\subsection{Classical Federated Optimization}
The original standard for distributed training across edge devices is Federated Averaging (FedAvg) \citep{McMahan2017_originalFL}, which reduces communication overhead by allowing clients to perform multiple local stochastic gradient descent (SGD) steps before server aggregation. However, when deployed across highly heterogeneous (non-IID) data distributions, the local objectives diverge from the global objective. This spatial heterogeneity causes the local models to drift toward disparate local optima, formally characterized as \emph{client drift} \citep{Zhao2018_related, Li2020_related}.

To mitigate client drift, two primary classes of advanced aggregation techniques have been developed. The first class employs proximal regularization and variance reduction. In details, FedProx \citep{Li2020_related} restricts local updates via a proximal term tied to the global model, while SCAFFOLD \citep{Karimireddy_Scaffold} introduces stateful control variates. These control variates maintain a memory of the gradient disparity between the server and the clients in order to correct the local update directions to emulate centralized training. The second class leverages stateful momentum tracking at the server level. Specifically, algorithms like FedAvgM \citep{Hsu2019_FedAvgM} and the FedOPT framework (including FedAdam and FedYogi) \citep{Reddi2021_related} demonstrate that maintaining an exponentially weighted moving average (EMA) of past global pseudo-gradients significantly stabilizes convergence over heterogeneous networks. 

Crucially, however, the theoretical guarantees of all these classical aggregation techniques rely on a strict mathematical premise, that is, the stochastic gradients evaluated by the clients must be unbiased estimators of their local objectives ($\mathbb E[g_i] = \nabla f_i$). While classical optimization easily satisfies this via standard mini-batch sampling, this assumption is fundamentally violated in the quantum regime. When client models are parameterized by noisy physical quantum circuits, the gradients are injected with an unavoidable, deterministic hardware bias. Classical control variates and momentum trackers are unequipped to handle this systemic bias, necessitating a new aggregation technique for Quantum Federated Learning.

\subsection{Quantum Federated Learning}

The integration of distributed optimization with quantum computing has rapidly formalized into the field of Quantum Federated Learning (QFL). Foundational works primarily established the feasibility of training Parameterized Quantum Circuits (PQCs) across decentralized nodes, using FedAvg~\citep{Chen2021_QFLrelated,Xia2021_QFLrelated}. Driven by its privacy-preserving properties and the enhanced representational capacity of quantum models, QFL has recently been deployed across diverse, security-critical domains. In literature, there are several applications of QFL in drug discovery~\citep{SAGGI2024_QFL_Drug,MESSINIS2025_QFL_Drug}, smart grids \citep{Ren2024_QFLrelated}, intelligent transportation systems \citep{Yamany2023_QFLrelated}, financial fraud detection \citep{Innan2025_QFLrelated}, and privacy-sensitive medical image classification \citep{Lusnig2024_QFLrelated}.

As QFL moves toward practical deployment, a growing line of work has begun to optimize the federated training process and address heterogeneity across clients. To improve local optimization, Qi et al.~\citep{Qi2023_QFLrelated} replaced standard stochastic gradient descent with Federated Quantum Natural Gradient Descent. To address statistical heterogeneity, Zhao~\citep{Zhao2023_QFLrelated} proposed a one-shot communication scheme for non-IID classical data based on quantum-channel decomposition. More recently, Han et al.~\citep{Han2026_relatedQFL} introduced a layerwise aggregation strategy for clients with heterogeneous quantum circuit depths. These works represent important steps toward practical QFL, but they do not fully address the issues that arise in realistic deployments. While Zhao~\citep{Zhao2023_QFLrelated} mitigates non-IID effects in a one-shot QFL construction, it does not study the accumulated client drift that arises in iterative QFL training under repeated local updates. On the other hand, Han et al.~\citep{Han2026_relatedQFL} addresses circuit-depth heterogeneity, but does not explicitly correct the systematic optimization bias induced by noisy quantum execution during gradient evaluation.

This leaves a critical mathematical gap. In realistic QFL, local updates can suffer from a \emph{double-drift} phenomenon. First, non-IID data induces \emph{statistical drift}, where each client's local gradient deviates from the global descent direction. Second, noisy quantum execution induces \emph{quantum drift}, where the gradient estimated on hardware is systematically biased away from the ideal client gradient. Finite-shot measurements further introduce stochastic variance on top of these two drift sources. While statistical drift has been studied in classical FL, it remains underdeveloped in QFL. Meanwhile, quantum drift is unique to noisy quantum hardware and is not addressed by any classical or quantum FL method. In this work, our framework targets this unaddressed double-drift problem by co-designing control-variate correction with quantum-noise-aware gradient modeling.

\section{Completed Technical Parts and Proofs}
\label{appendix:theory}

In this section, we first analyze the behavior of vanilla FedAvg under quantum-noisy gradient oracles, and then study how an anchor-corrected update can mitigate the non-vanishing hardware-bias floor. Our first result shows that standard quantum FedAvg converges only to a neighborhood of stationarity that contains a constant bias floor induced by hardware noise. We then introduce a ZNE-based anchor correction and show that, under an idealized unbiased anchor oracle, the absolute hardware bias can be replaced by a drift-coupled residual that vanishes as the local update radius shrinks.

\subsection{Problem setup}

We consider federated optimization of
\begin{align}
f(x) := \frac{1}{N}\sum_{i=1}^N f_i(x),
\end{align}
where $N$ is the total number of clients, $f_i$ is the local objective of client $i$, and $x \in \mathbb{R}^d$ denotes the model parameters.

At communication round $r$, the server samples a subset $\mathcal S_r$ of $S$ clients uniformly without replacement. Each selected client performs $K$ local steps with local stepsize $\eta_\ell$, and the server uses a global stepsize $\eta_g$. For convenience, define the effective stepsize
\begin{align}
\widetilde{\eta} := \eta_g \eta_\ell K.
\end{align}
We let $\mathcal F_{r-1}$ denote the filtration up to the beginning of round $r$.

\subsection{Assumptions}

\begin{assumption}[Smoothness]
Each local objective $f_i$ is $\beta$-smooth:
\begin{align}
\|\nabla f_i(x)-\nabla f_i(y)\| \le \beta \|x-y\|,
\qquad \forall x,y \in \mathbb R^d.
\end{align}
Hence, the global objective $f$ is also $\beta$-smooth.
\end{assumption}

\begin{assumption}[Bounded gradient dissimilarity]
There exist constants $G \ge 0$ and $B \ge 1$ such that for all $x$,
\begin{align}
\frac{1}{N}\sum_{i=1}^N \|\nabla f_i(x)\|^2
\le
G^2 + B^2 \|\nabla f(x)\|^2.
\end{align}
In particular, this implies
\begin{align}
\frac{1}{N}\sum_{i=1}^N \|\nabla f_i(x)-\nabla f(x)\|^2
\le
G^2 + (B^2-1)\|\nabla f(x)\|^2.
\end{align}
\end{assumption}

\begin{assumption}[Stochastic data oracle]
For each client $i$, the noise-free stochastic gradient $g_i(x;\xi)$ satisfies
\begin{align}
\mathbb E_\xi[g_i(x;\xi)\mid x] = \nabla f_i(x),
\end{align}
and
\begin{align}
\mathbb E_\xi\!\left[\|g_i(x;\xi)-\nabla f_i(x)\|^2 \mid x\right] \le \sigma^2.
\end{align}
\end{assumption}

\begin{assumption}[Biased quantum gradient oracle]
For each client $i$, the hardware-executed gradient estimator $\widetilde g_i(x;\xi)$ satisfies
\begin{align}
\mathbb E_q[\widetilde g_i(x;\xi)\mid x,\xi]
=
g_i(x;\xi) + b_{q,i}(x),
\end{align}
where the hardware bias obeys
\begin{align}
\|b_{q,i}(x)\| \le U_q,
\qquad \forall i,x.
\end{align}
In addition, the conditional quantum variance is bounded:
\begin{align}
\mathbb E_q\!\left[\left\|\widetilde g_i(x;\xi)-\mathbb E_q[\widetilde g_i(x;\xi)\mid x,\xi]\right\|^2 \,\middle|\, x,\xi \right]
\le \sigma_q^2.
\end{align}
Consequently, by total variance, the conditional second moment of the centered oracle noise is bounded by $\sigma^2+\sigma_q^2$.
\end{assumption}

\begin{assumption}[Lipschitz quantum bias]
For each client $i$, the hardware bias map is $L_q$-Lipschitz:
\begin{align}
\|b_{q,i}(x)-b_{q,i}(y)\| \le L_q \|x-y\|,
\qquad \forall x,y \in \mathbb R^d.
\end{align}
\end{assumption}
\noindent\textit{Remark.} This assumption is naturally satisfied for certain common quantum error channels. For instance, under a global depolarizing channel, the hardware gradient scales as an attenuated version of the true gradient, meaning the bias is proportional to the gradient itself and inherits Lipschitz continuity directly from Assumption~\ref{ass:smooth}.

We also assume $f^\star > -\infty$, where $f^\star := \inf_x f(x)$.

\subsection{FedAvg under QFL}
\label{appendix:fedavg_proof}
We first analyze vanilla FedAvg in the presence of quantum bias and variance. Our goal is to derive a standard nonconvex stationarity bound for a uniformly sampled iterate. The main challenge, compared with classical FedAvg, is that the local quantum gradient estimator is not exactly unbiased: besides stochastic variance, it also contains a persistent hardware bias. This bias interacts with client drift and must be tracked carefully.

At a high level, the proof follows the usual smoothness-based descent template. For one communication round, smoothness of $f$ gives
\begin{align}
\label{eq:fedavg_smooth_start}
\mathbb E[f(x^r)\mid \mathcal F_{r-1}]
\le
f(x^{r-1})
+
\left\langle
\nabla f(x^{r-1}),
\mathbb E[\Delta x^r \mid \mathcal F_{r-1}]
\right\rangle
+
\frac{\beta}{2}
\mathbb E[\|\Delta x^r\|^2\mid \mathcal F_{r-1}],
\end{align}
where \(\Delta x^r := x^r-x^{r-1}\). Thus, to prove descent, it suffices to control two quantities:

\begin{itemize}
    \item the \emph{descent inner product}
    \(
    \left\langle
    \nabla f(x^{r-1}),
    \mathbb E[\Delta x^r \mid \mathcal F_{r-1}]
    \right\rangle,
    \)
    which should be negative and proportional to \(-\|\nabla f(x^{r-1})\|^2\);
    \item the \emph{second moment}
    \(
    \mathbb E[\|\Delta x^r\|^2\mid \mathcal F_{r-1}],
    \)
    which appears as the quadratic smoothness penalty.
\end{itemize}

Both terms are affected by the fact that clients take \(K\) local steps before averaging, so the stochastic gradients are evaluated at local points \(y_{i,k}^r\) that drift away from the synchronization point \(x^{r-1}\). To track this effect, define the average local-drift energy
\begin{align*}
E_r
:=
\frac{1}{KN}\sum_{i=1}^N\sum_{k=0}^{K-1}
\mathbb E\!\left[
\|y_{i,k}^r-x^{r-1}\|^2
\,\middle|\,
\mathcal F_{r-1}
\right].
\end{align*}

The proof proceeds in three steps. First, we show that the expected update still points roughly along the negative gradient, up to drift and quantum bias. Second, we bound the second moment of the server update. Third, we derive a recursion for \(E_r\) and substitute it back into the one-round descent inequality.

Vanilla FedAvg performs local updates
\begin{align*}
y_{i,k+1}^r
=
y_{i,k}^r - \eta_\ell \widetilde g_i(y_{i,k}^r;\xi_{i,k}^r),
\qquad
k=0,\dots,K-1,
\end{align*}
with initialization \(y_{i,0}^r=x^{r-1}\), and the server update is
\begin{align*}
x^r
=
x^{r-1}
+
\eta_g \cdot \frac{1}{S}\sum_{i\in\mathcal S_r}(y_{i,K}^r-x^{r-1})
=
x^{r-1}-\widetilde\eta \widehat g^r,
\end{align*}
where
\begin{align*}
\widehat g^r
:=
\frac{1}{SK}\sum_{i\in\mathcal S_r}\sum_{k=0}^{K-1}
\widetilde g_i(y_{i,k}^r;\xi_{i,k}^r),
\qquad
\widetilde\eta := \eta_g\eta_\ell K.
\end{align*}

To establish the final stationary bound, we decompose the analysis into a sequence of supporting lemmas.

\paragraph{Step I: we aim to control the descent direction.}
The first lemma shows that the expected server update is aligned with the negative population gradient, up to two error sources: local drift and quantum bias.

\begin{lemma}[Descent direction]
\label{lem:fedavg_direction}
Under Assumptions~\ref{ass:smooth}, \ref{ass:data}, and \ref{ass:quantum},
\begin{align*}
\left\langle
\nabla f(x^{r-1}),
\mathbb E[\Delta x^r \mid \mathcal F_{r-1}]
\right\rangle
\le
-\frac{\widetilde\eta}{2}\|\nabla f(x^{r-1})\|^2
+
\widetilde\eta \beta^2 E_r
+
\widetilde\eta U_q^2.
\end{align*}
\end{lemma}

\begin{proof}
By definition, we have:
$\Delta x^r = -\widetilde\eta \widehat g^r
$. Let
\begin{align*}
h^r
:=
\frac{1}{KN}\sum_{i=1}^N\sum_{k=0}^{K-1}
\mathbb E\!\left[
\nabla f_i(y_{i,k}^r)+b_{q,i}(y_{i,k}^r)
\,\middle|\,
\mathcal F_{r-1}
\right].
\end{align*}
Using uniform client sampling and Assumption~\ref{ass:quantum}, we have:
$$
\mathbb E[\Delta x^r\mid \mathcal F_{r-1}]
=
-\widetilde\eta\, h^r.
$$
Therefore,
$$
\left\langle
\nabla f(x^{r-1}),
\mathbb E[\Delta x^r \mid \mathcal F_{r-1}]
\right\rangle
=
-\widetilde\eta \langle \nabla f(x^{r-1}), h^r\rangle.
$$

We now compare $h^r$ with $\nabla f(x^{r-1})$. By Jensen's inequality,
\begin{align*}
\|h^r-\nabla f(x^{r-1})\|^2
&\le
\frac{1}{KN}\sum_{i=1}^N\sum_{k=0}^{K-1}
\mathbb E\!\left[
\left\|
\bigl(\nabla f_i(y_{i,k}^r)-\nabla f_i(x^{r-1})\bigr)
+
b_{q,i}(y_{i,k}^r)
\right\|^2
\middle|
\mathcal F_{r-1}
\right]
\nonumber\\
&\le
2\beta^2 E_r + 2U_q^2.
\end{align*}
Finally, applying
$
-\langle a,b\rangle
\le
-\frac12\|a\|^2 + \frac12\|a-b\|^2
$ with $a=\nabla f(x^{r-1})$ and $b=h^r$, we achieve the result.
\end{proof}

\paragraph{Step II: we aim to control the quadratic penalty.}
The next lemma bounds the size of the server update. This is needed for the smoothness penalty in \eqref{eq:fedavg_smooth_start}.

\begin{lemma}[Second moment of the update]
\label{lem:fedavg_second_moment_clean}
Under Assumptions~\ref{ass:smooth}--\ref{ass:quantum}, there exists an absolute constant $C>0$ such that
\begin{align*}
\mathbb E[\|\Delta x^r\|^2 \mid \mathcal F_{r-1}]
\le
C \widetilde\eta^2
\left(
\left(1+\frac{B^2}{S}\right)\|\nabla f(x^{r-1})\|^2
+
\frac{G^2}{S}
+
\beta^2 E_r
+
U_q^2
+
\frac{\sigma^2+\sigma_q^2}{KS}
\right).
\end{align*}
\end{lemma}

\begin{proof}
Since $\Delta x^r = -\widetilde\eta\,\widehat g^r$, we aim to bound $\mathbb E[\|\widehat g^r\|^2 \mid \mathcal F_{r-1}]$. We decompose the global update direction $\widehat g^r = \frac{1}{SK}\sum_{i\in\mathcal S_r}\sum_{k=0}^{K-1} \widetilde g_i(y_{i,k}^r;\xi_{i,k}^r)$ into four components: the true gradient evaluated at the anchor $x^{r-1}$, the local drift, the hardware bias, and the stochastic quantum noise:
\begin{align*}
\widehat g^r
&=
\underbrace{
\frac{1}{S}\sum_{i\in\mathcal S_r}\nabla f_i(x^{r-1})
}_{A_1}
+
\underbrace{
\frac{1}{SK}\sum_{i\in\mathcal S_r}\sum_{k=0}^{K-1}
\bigl(\nabla f_i(y_{i,k}^r)-\nabla f_i(x^{r-1})\bigr)
}_{A_2} \\
&+
\underbrace{
\frac{1}{SK}\sum_{i\in\mathcal S_r}\sum_{k=0}^{K-1}
b_{q,i}(y_{i,k}^r)
}_{A_3}
+
\underbrace{
\frac{1}{SK}\sum_{i\in\mathcal S_r}\sum_{k=0}^{K-1}
\zeta_{i,k}^r
}_{A_4},
\end{align*}

where $\zeta_{i,k}^r := \widetilde g_i(y_{i,k}^r;\xi_{i,k}^r) - \nabla f_i(y_{i,k}^r) - b_{q,i}(y_{i,k}^r)$ is the zero-mean noise. Applying the standard expansion $\|\sum_{m=1}^4 A_m\|^2 \le 4\sum_{m=1}^4 \|A_m\|^2$, we conditionally bound the expected squared norm of each term:

1. \textbf{Population Sampling ($A_1$):} We split the sampled gradient into the global gradient and the sampling deviation. By applying standard uniform sampling bounds and the bounded-dissimilarity condition (Assumption~\ref{ass:dissim}), we obtain:
\begin{align*}
    \mathbb E[\|A_1\|^2 \mid \mathcal F_{r-1}]
&\le
2\|\nabla f(x^{r-1})\|^2
+
\frac{2}{S}\frac{1}{N}\sum_{i=1}^N \|\nabla f_i(x^{r-1})-\nabla f(x^{r-1})\|^2\\
&\le
C\left( \left(1+\frac{B^2}{S}\right)\|\nabla f(x^{r-1})\|^2 + \frac{G^2}{S} \right).
\end{align*}

2. \textbf{Local Drift ($A_2$):} By Jensen's inequality and the $\beta$-smoothness of the local objectives, the expected norm is bounded by the expected squared distance to the anchor. Averaging over the uniform subset $\mathcal S_r$, we have
$$
\mathbb E[\|A_2\|^2 \mid \mathcal F_{r-1}] 
\le 
\frac{\beta^2}{SK} \sum_{i\in\mathcal S_r}\sum_{k=0}^{K-1} \mathbb E[\|y_{i,k}^r-x^{r-1}\|^2 \mid \mathcal F_{r-1}] 
= 
\beta^2 E_r.
$$

3. \textbf{Hardware Bias ($A_3$):} By Jensen's inequality and the uniform bias bound (Assumption~\ref{ass:quantum}), we strictly have:
$$
\mathbb E[\|A_3\|^2 \mid \mathcal F_{r-1}] \le U_q^2.
$$

4. \textbf{Stochastic Noise ($A_4$):} Because the zero-mean noise terms $\zeta_{i,k}^r$ are conditionally independent across distinct clients and steps, their cross-terms vanish in expectation. The variance strictly averages over the $SK$ samples:
$$
\mathbb E[\|A_4\|^2 \mid \mathcal F_{r-1}]
=
\frac{1}{S^2K^2}\sum_{i\in\mathcal S_r}\sum_{k=0}^{K-1} \mathbb E[\|\zeta_{i,k}^r\|^2 \mid \mathcal F_{r-1}]
\le
\frac{\sigma^2+\sigma_q^2}{SK}.
$$

Substituting these four component bounds back into the expansion inequality and absorbing all absolute numerical multipliers into $C$, we achieve the result.
\end{proof}

\paragraph{Step III: we aim to control the drift term.}
The previous two lemmas show that the only remaining obstacle is the drift quantity $E_r$. We now show that $E_r$ is small when the local step size is small enough.

\begin{lemma}[Local-drift bound]
\label{lem:fedavg_drift_clean}
Assume $\eta_\ell \le \frac{c_0}{(1+B^2)\beta K \eta_g}$ and $\eta_g \ge 1$, for a sufficiently small absolute constant $c_0>0$. Then there exists an absolute constant $C>0$ such that
\begin{align*}
E_r
\le
C \eta_\ell^2 K^2
\left(
B^2\|\nabla f(x^{r-1})\|^2
+
G^2
+
U_q^2
+
\sigma^2+\sigma_q^2
\right).
\end{align*}
\end{lemma}

\begin{proof}
For each client $i$ and local step $k$, let $D_{i,k}^r := \mathbb E[\|y_{i,k}^r-x^{r-1}\|^2 \mid \mathcal F_{r-1}]$ denote the expected squared drift from the anchor. Unrolling the local update rule $y_{i,k}^r = y_{i,k-1}^r - \eta_\ell \widetilde g_i(y_{i,k-1}^r;\xi_{i,k-1}^r)$, we have:
$$
y_{i,k}^r-x^{r-1}
=
-\eta_\ell \sum_{j=0}^{k-1}\widetilde g_i(y_{i,j}^r;\xi_{i,j}^r).
$$
Applying the standard expansion inequality $\|\sum_{j=0}^{k-1} v_j\|^2 \le K \sum_{j=0}^{k-1}\|v_j\|^2$ and taking the conditional expectation gives:
\begin{align}
D_{i,k}^r
&\le
\eta_\ell^2 K \sum_{j=0}^{k-1} \mathbb E\!\left[\|\widetilde g_i(y_{i,j}^r;\xi_{i,j}^r)\|^2 \mid \mathcal F_{r-1}\right].
\label{eq:drift_first_step}
\end{align}

We bound the second moment of the local gradient by decomposing it into the true local gradient, the hardware bias, and the zero-mean stochastic noise. Applying the expansion inequality to these three terms, we have:
$$
\mathbb E\!\left[\|\widetilde g_i(y_{i,j}^r;\xi_{i,j}^r)\|^2 \mid \mathcal F_{r-1}\right]
\le
3\mathbb E\!\left[\|\nabla f_i(y_{i,j}^r)\|^2 \mid \mathcal F_{r-1}\right]
+
3U_q^2
+
3(\sigma^2+\sigma_q^2).
$$
To relate the local gradient back to the anchor, we add and subtract $\nabla f_i(x^{r-1})$ and apply $\beta$-smoothness, which bounds the gradient norm by $2\beta^2 \|y_{i,j}^r-x^{r-1}\|^2 + 2\|\nabla f_i(x^{r-1})\|^2$. Substituting this into the second moment bound and absorbing the numerical multipliers into an absolute constant $C$ yields:
\begin{align}
\mathbb E\!\left[\|\widetilde g_i(y_{i,j}^r;\xi_{i,j}^r)\|^2 \mid \mathcal F_{r-1}\right]
&\le
C\beta^2 D_{i,j}^r
+
C
\left(
\|\nabla f_i(x^{r-1})\|^2
+
U_q^2
+
\sigma^2+\sigma_q^2
\right).
\label{eq:local_grad_bound_final}
\end{align}

Substituting \eqref{eq:local_grad_bound_final} into \eqref{eq:drift_first_step}, we have the recursive inequality:
\begin{align}
D_{i,k}^r
&\le
C\beta^2\eta_\ell^2 K \sum_{j=0}^{k-1} D_{i,j}^r
+
C\eta_\ell^2 K^2
\left(
\|\nabla f_i(x^{r-1})\|^2
+
U_q^2
+
\sigma^2+\sigma_q^2
\right).
\label{eq:drift_recursion_compact}
\end{align}

Under the step-size assumption, the recursive contraction factor is strictly bounded as $C\beta^2\eta_\ell^2 K^2 \le C c_0^2$. By choosing the absolute constant $c_0$ sufficiently small, this factor is restricted to $\mathcal{O}(1)$. Applying a standard discrete Grönwall argument resolves the recursion uniformly for all $k \le K$:
\begin{align}
D_{i,k}^r
\le
C\eta_\ell^2 K^2
\left(
\|\nabla f_i(x^{r-1})\|^2
+
U_q^2
+
\sigma^2+\sigma_q^2
\right).
\label{eq:drift_uniform_bound}
\end{align}

To compute the population drift energy $E_r$, we average $D_{i,k}^r$ over all clients $i \in \{1,\dots,N\}$ and steps $k$. Applying the bounded dissimilarity assumption (Assumption~\ref{ass:dissim}) which states that $\frac{1}{N}\sum_{i=1}^N \|\nabla f_i(x^{r-1})\|^2 \le B^2\|\nabla f(x^{r-1})\|^2 + G^2$, we achieve the result.
\end{proof}

\paragraph{Step IV: we combine the pieces into one-round descent.}
We now substitute the previous three lemmas into the smoothness inequality \eqref{eq:fedavg_smooth_start}.
\begin{lemma}[One-round descent]
\label{lem:fedavg_one_round_final}
Assume $\eta_\ell \le \frac{c_0}{(1+B^2)\beta K \eta_g}$ and $\eta_g \ge 1$ for a sufficiently small absolute constant $c_0>0$. Then there exist absolute constants $c, C > 0$ such that
\begin{align*}
\mathbb E[f(x^r)\mid \mathcal F_{r-1}]
\le\;&
f(x^{r-1})
-
c \widetilde\eta \|\nabla f(x^{r-1})\|^2
\notag\\
&+
C \beta \widetilde\eta^2
\left(
\frac{G^2}{S}
+
\frac{\sigma^2+\sigma_q^2}{KS}
+
U_q^2
\right)
\notag\\
&+
C \widetilde\eta \beta^2 \eta_\ell^2 K^2
\left(
G^2+\sigma^2+\sigma_q^2+U_q^2
\right)
+
C \widetilde\eta U_q^2.
\end{align*}
\end{lemma}

\begin{proof}
By the $\beta$-smoothness of the objective $f$, the expected function value is bounded by:
\begin{align*}
\mathbb E[f(x^r)\mid \mathcal F_{r-1}]
\le
f(x^{r-1})
+
\left\langle \nabla f(x^{r-1}), \mathbb E[\Delta x^r\mid \mathcal F_{r-1}] \right\rangle
+
\frac{\beta}{2}\mathbb E[\|\Delta x^r\|^2\mid \mathcal F_{r-1}].
\end{align*}

We compile the descent components into this inequality. Substituting the first-order expected direction (Lemma~\ref{lem:fedavg_direction}) and the second-moment penalty (Lemma~\ref{lem:fedavg_second_moment_clean}), using a generic absolute constant $C$, we have:
\begin{align*}
\mathbb E[f(x^r)\mid \mathcal F_{r-1}]
\le\;&
f(x^{r-1})
-
\frac{\widetilde\eta}{2}\|\nabla f(x^{r-1})\|^2
+
\widetilde\eta \beta^2 E_r
+
\widetilde\eta U_q^2
\notag\\
&+
C \beta \widetilde\eta^2
\left(
\left(1+\frac{B^2}{S}\right)\|\nabla f(x^{r-1})\|^2
+
\frac{G^2}{S}
+
\beta^2 E_r
+
U_q^2
+
\frac{\sigma^2+\sigma_q^2}{KS}
\right).
\end{align*}

Substituting the bound for $E_r$ from Lemma~\ref{lem:fedavg_drift_clean} into the drift penalty $\mathcal{O}(\widetilde\eta \beta^2 E_r)$ and grouping the terms proportional to $\|\nabla f(x^{r-1})\|^2$, we have the descent component:
$$
-\widetilde\eta \underbrace{\Biggl[
\frac{1}{2}
-
C \beta \widetilde\eta \left(1+\frac{B^2}{S}\right)
-
C \beta^2 \eta_\ell^2 K^2 B^2
\Biggr]}_{:= \rho} \|\nabla f(x^{r-1})\|^2.
$$

Applying the step-size conditions $\widetilde\eta \le \frac{c_0}{(1+B^2)\beta}$ and $\eta_g \ge 1$ (which strictly implies $\eta_\ell K \le \frac{c_0}{(1+B^2)\beta}$), the subtracted terms are formally bounded:
$$
C \beta \widetilde\eta \left(1+\frac{B^2}{S}\right) \le C c_0
\qquad \text{and} \qquad
C \beta^2 \eta_\ell^2 K^2 B^2 \le C c_0^2.
$$
By choosing $c_0>0$ sufficiently small, we guarantee the bracket evaluates to a strict absolute constant $\rho \ge c > 0$. 

Finally, collecting the remaining gradient-independent residuals from the $E_r$ substitution and the second-moment bound directly, we have:
$$
C \widetilde\eta \beta^2 \eta_\ell^2 K^2
\left(
G^2+\sigma^2+\sigma_q^2+U_q^2
\right)
+
C \beta \widetilde\eta^2
\left(
\frac{G^2}{S}
+
\frac{\sigma^2+\sigma_q^2}{KS}
+
U_q^2
\right)
+
C \widetilde\eta U_q^2.
$$
Adding this residual sum to the strict descent term $-c\widetilde\eta \|\nabla f(x^{r-1})\|^2$ concludes the proof.
\end{proof}

\paragraph{Step V: we telescope over rounds.}
The final theorem follows by summing the one-round descent inequality over $r=1,\dots,R$ and averaging.

\begin{theorem}[FedAvg stationarity under quantum bias]
\label{thm:fedavg_final_appendix}
Suppose Assumptions~\ref{ass:smooth}--\ref{ass:quantum} hold. Let $F := f(x^0)-f^\star$. Run quantum FedAvg for $R$ rounds, and let $\bar x^R$ be chosen uniformly from $\{x^0,\dots,x^{R-1}\}$. Assume
$$
\eta_\ell \le \frac{c_0}{(1+B^2)\beta K \eta_g}, \qquad \eta_g \ge 1,
$$
for a sufficiently small absolute constant $c_0>0$. Then, with $\widetilde\eta := \eta_g \eta_\ell K$, we have:
$$
\mathbb E\|\nabla f(\bar x^R)\|^2
\le
\mathcal O\left(
\frac{F}{\widetilde\eta R}
+
\beta \widetilde\eta \left(\frac{G^2}{S} + \frac{\sigma^2+\sigma_q^2}{KS}\right)
+
\beta^2 \eta_\ell^2 K^2 \left(G^2+\sigma^2+\sigma_q^2\right)
+
U_q^2
\right).
$$
\end{theorem}

\begin{proof}
Taking the full expectation of the one-round descent bound (Lemma~\ref{lem:fedavg_one_round_final}), summing over rounds $r=1,\dots,R$, and rearranging to isolate the expected gradient norm, we have the telescoping sum:
\begin{align*}
c \widetilde\eta \sum_{r=1}^R \mathbb E\|\nabla f(x^{r-1})\|^2
&\le
f(x^0) - \mathbb E[f(x^R)] \notag \\
&\quad+
R C \Biggl[ \beta \widetilde\eta^2 \left( \frac{G^2}{S} + \frac{\sigma^2+\sigma_q^2}{KS} + U_q^2 \right) \notag \\
&\qquad+ \widetilde\eta \beta^2 \eta_\ell^2 K^2 \left( G^2+\sigma^2+\sigma_q^2+U_q^2 \right) + \widetilde\eta U_q^2 \Biggr].
\end{align*}

Since the objective is globally bounded below by $f^\star$, we strictly have $f(x^0) - \mathbb E[f(x^R)] \le f(x^0) - f^\star = F$. Substituting this bound and dividing the entire aggregated inequality by $R c \widetilde\eta$, we have:
\begin{align*}
\frac{1}{R}\sum_{r=1}^R \mathbb E\|\nabla f(x^{r-1})\|^2
&\le
\frac{F}{c \widetilde\eta R}
+
\frac{C}{c} \Biggl[ \beta \widetilde\eta \left( \frac{G^2}{S} + \frac{\sigma^2+\sigma_q^2}{KS} + U_q^2 \right) \notag \\
&\qquad+ \beta^2 \eta_\ell^2 K^2 \left(G^2+\sigma^2+\sigma_q^2+U_q^2\right) + U_q^2 \Biggr].
\end{align*}

Because the output $\bar x^R$ is sampled uniformly at random from the $R$ prior iterates, its expected gradient norm exactly equals the left side of this inequality. Finally, because the step-size conditions strictly enforce $\beta \widetilde\eta \le \mathcal{O}(1)$ and $\beta \eta_\ell K \le \mathcal{O}(1)$, the generated higher-order bias terms $\mathcal{O}(\beta \widetilde\eta U_q^2)$ and $\mathcal{O}(\beta^2 \eta_\ell^2 K^2 U_q^2)$ are strictly dominated by the unmitigated residual $\mathcal{O}(U_q^2)$. Absorbing these remaining residuals and the absolute constants $C$ and $c$ into the $\mathcal O(\cdot)$ notation yields the stated compact theorem bound.
\end{proof}

\subsection{Quantum Federated ZNE-Anchored Control (\NAMEA)}
\label{appendix:qfedzac_proof}

The deployment of QFL via standard FedAvg suffers from a fundamental \emph{double drifting problem}. First, classical data heterogeneity pulls individual clients toward spurious local minima, causing classical client drift. Second, NISQ devices introduce input-dependent quantum hardware bias, which systematically misguides the local gradients away from the true analytical landscape. These two forces compound, establishing unresolvable error floors. To solve this double drifting problem, we propose Quantum Federated ZNE-Anchored Control (\NAMEA). \NAMEA\ is a federated optimization framework for QFL where client-side control variates are updated using the standard, biased quantum oracle, while the server-side reference control is constructed by averaging ZNE-corrected anchor gradients. This design perfectly cancels the dual pull of local data and hardware biases by preserving inexpensive local updates and injecting a highly accurate global correction.

\paragraph{The Role of ZNE in \NAMEA.}
Before presenting the proofs, we clarify how Zero-Noise Extrapolation (ZNE) alters the optimization landscape. At the algorithmic level, ZNE enforces a strict physical trade-off between deterministic error and stochastic noise:
\begin{itemize}
    \item \textbf{Bias Suppression ($\kappa_b>1$):} ZNE successfully mitigates systemic hardware error, reducing the residual bias from $U_q$ down to $U_q/\kappa_b$.
    \item \textbf{Variance Amplification ($\kappa_v>1$):} Because ZNE requires linearly combining independent quantum shots, it inevitably inflates the underlying quantum noise. The variance floor increases from $\sigma^2+\sigma_q^2$ to $\sigma^2+\kappa_v\sigma_q^2$.
\end{itemize}

Because ZNE requires multiple circuit evaluations at various noise scales, it is highly resource-intensive. \NAMEA\ circumvents this hardware bottleneck through its asymmetric architecture. Rather than burdening edge clients with costly ZNE evaluations for every local step, clients use cheap, unmitigated oracles. ZNE is strictly reserved for the server's \emph{global reference control}.
Crucially, because the server mathematically averages these ZNE-corrected evaluations across the entire population $N$ using stateful tracking, the amplified variance penalty is cleanly suppressed by a $1/N$ factor. This ensures the global convergence path benefits from the mitigated bias ($U_q/\kappa_b$) without being derailed by the amplified noise, while the edge devices operate efficiently at the unmitigated bounds.


\paragraph{Algorithm.}
At round $r$, the server samples a subset $\mathcal S_r$ of $S$ clients uniformly without replacement. Let $c_i^{r-1}\in\mathbb R^d$ denote the local control variate of client $i$, and let $c_{\mathrm{srv}}^{r-1}\in\mathbb R^d$ denote the server reference control.

Each selected client initializes at $y_{i,0}^r=x^{r-1}$ and performs $K$ local updates:
\begin{align}
y_{i,k+1}^r
&=
y_{i,k}^r
-
\eta_\ell
\Bigl(
\widetilde g_i(y_{i,k}^r;\xi_{i,k}^r)
-
c_i^{r-1}
+
c_{\mathrm{srv}}^{r-1}
\Bigr),
\qquad
k=0,\dots,K-1.
\label{eq:asym_scaffold_local_final2}
\end{align}
The server update is
\begin{align*}
x^r
&=
x^{r-1}
+
\eta_g\cdot \frac{1}{S}\sum_{i\in\mathcal S_r}(y_{i,K}^r-x^{r-1}).
\end{align*}
Equivalently, with $\widetilde\eta:=\eta_g\eta_\ell K$,
\begin{align*}
x^r
&=
x^{r-1}-\widetilde\eta\,\widehat v^r,
\qquad
\widehat v^r
:=
\frac{1}{SK}\sum_{i\in\mathcal S_r}\sum_{k=0}^{K-1}
\Bigl(
\widetilde g_i(y_{i,k}^r;\xi_{i,k}^r)-c_i^{r-1}+c_{\mathrm{srv}}^{r-1}
\Bigr).
\end{align*}

For participating clients, the local control variate is updated using an exponential moving average of the standard biased oracle evaluated over a random batch, with momentum parameter $\alpha \in (0, 1]$:
\begin{align*}
c_i^r
& =
(1-\alpha)c_i^{r-1} + \alpha \widetilde g_i(x^{r-1};\xi_{i}^r),
\qquad
i\in\mathcal S_r.
\end{align*}
For unsampled clients, we keep $c_i^r=c_i^{r-1}$.
In addition, for each sampled client $i\in\mathcal S_r$, the server forms a ZNE-corrected version of the client control variate:
\begin{align*}
c_{i,\mathrm{ZNE}}^r
&=
(1-\alpha)c_{i,\mathrm{ZNE}}^{r-1} + \alpha \widetilde g_i^{\mathrm{ZNE}}(x^{r-1};\xi_{i}^r),
\qquad
i\in\mathcal S_r.
\end{align*}
The server reference control is then defined as the average of these sampled ZNE-corrected client controls:
\begin{align}
c_{\mathrm{srv}}^r
&=
c_{\mathrm{srv}}^{r-1} 
+ 
\frac{1}{N}\sum_{i\in\mathcal S_r} \Bigl( c_{i,\mathrm{ZNE}}^r - c_{i,\mathrm{ZNE}}^{r-1} \Bigr).
\label{eq:asym_server_cv_final2}
\end{align}

\paragraph{Oracle models.}
For the fresh local gradients and biased local controls, we use the standard biased oracle:
\begin{align*}
\mathbb E[\widetilde g_i(x;\xi)\mid x]
=
\nabla f_i(x)+b_{q,i}(x),
\qquad
\|b_{q,i}(x)\|\le U_q,
\\
\mathbb E\!\left[
\left\|
\widetilde g_i(x;\xi)-\mathbb E[\widetilde g_i(x;\xi)\mid x]
\right\|^2
\,\middle|\,x
\right]
\le
\sigma^2+\sigma_q^2.
\end{align*}

For the ZNE-corrected client anchor oracle, we have
\begin{align}
\mathbb E[\widetilde g_i^{\mathrm{ZNE}}(x;\bar\zeta)\mid x]
=
\nabla f_i(x)+b_{q,i}^{\mathrm{ZNE}}(x),
\qquad
\|b_{q,i}^{\mathrm{ZNE}}(x)\|\le \frac{U_q}{\kappa_b},
\\
\mathbb E\!\left[
\left\|
\widetilde g_i^{\mathrm{ZNE}}(x;\bar\zeta)
-
\mathbb E[\widetilde g_i^{\mathrm{ZNE}}(x;\bar\zeta)\mid x]
\right\|^2
\,\middle|\,x
\right]
\le
\sigma^2+\kappa_v\sigma_q^2,
\label{eq:asym_zne_assump_final2}
\end{align}
where $\kappa_b,\kappa_v>1$.

We define the local-drift energy
\begin{align*}
E_r
:=
\frac{1}{KN}\sum_{i=1}^N\sum_{k=0}^{K-1}
\mathbb E\!\left[
\|y_{i,k}^r-x^{r-1}\|^2
\right].
\end{align*}
and the dual control-estimation errors:
\begin{align*}
\Gamma_r
&:=
\frac{1}{N}\sum_{i=1}^N
\mathbb E\!\left[
\|c_i^r-\mu_i(x^r)\|^2
\right],
\qquad
\mu_i(x):=\nabla f_i(x)+b_{q,i}(x). \\
\Gamma_r^{\mathrm{ZNE}}
&:=
\frac{1}{N}\sum_{i=1}^N \mathbb E\!\left[ \|c_{i,\mathrm{ZNE}}^r-\mu_{i,\mathrm{ZNE}}(x^r)\|^2 \right],
\qquad
\mu_{i,\mathrm{ZNE}}(x):=\nabla f_i(x)+b_{q,i}^{\mathrm{ZNE}}(x).
\end{align*}
In addition, we define the server-client control mismatch
\begin{align*}
\Lambda_{r-1}
:=
\mathbb E\!\left[
\left\|
c_{\mathrm{srv}}^{r-1}-\frac{1}{N}\sum_{i=1}^N c_i^{r-1}
\right\|^2
\right].
\end{align*}

As before, smoothness gives
\begin{align}
\mathbb E[f(x^r)]
\le
\mathbb E[f(x^{r-1})]
+
\mathbb E\left\langle
\nabla f(x^{r-1}),\Delta x^r
\right\rangle
+
\frac{\beta}{2}\mathbb E[\|\Delta x^r\|^2],
\label{eq:asym_smooth_final2}
\end{align}
where $\Delta x^r:=x^r-x^{r-1}$.

\paragraph{Step I: descent direction.}
Because the client controls and the fresh local gradients use the same biased oracle at the anchor point, the anchor-point hardware bias cancels at first order, leaving only drift terms and the smaller ZNE server bias.

\begin{lemma}[Descent direction under asymmetric controls]
\label{lem:asym_direction_final2}
Under Assumptions~\ref{ass:smooth}, \ref{ass:data}, \ref{ass:quantum}, \ref{ass:lipschitz_bias}, and given a constant $C$, we have:
\begin{align}
\mathbb E\left\langle
\nabla f(x^{r-1}),\Delta x^r
\right\rangle
\le\;&
-\frac{\widetilde\eta}{2}\mathbb E\|\nabla f(x^{r-1})\|^2
+
C\widetilde\eta(\beta^2+L_q^2)E_r
\notag\\
&+
C\widetilde\eta \Gamma_{r-1}
+ C\widetilde\eta \Gamma_{r-1}^{\mathrm{ZNE}}+
C\widetilde\eta \left( \frac{U_q^2}{\kappa_b^2} \right).
\end{align}
\end{lemma}

\begin{proof}
Let
$$
h^r
:=
\frac{1}{KN}\sum_{i=1}^N\sum_{k=0}^{K-1} \mu_i(y_{i,k}^r)
-
\frac{1}{N}\sum_{i=1}^N c_i^{r-1}
+
c_{\mathrm{srv}}^{r-1},
\qquad
\mu_i(x):=\nabla f_i(x)+b_{q,i}(x).
$$
We have $
\mathbb E[\Delta x^r] = -\widetilde\eta\,\mathbb E[h^r].
$

First of all, we aim to bound $\mathbb E\|h^r-\nabla f(x^{r-1})\|^2$. By adding and subtracting $\frac{1}{N}\sum_{i=1}^N \mu_i(x^{r-1})$, we obtain:
\begin{align}
h^r - \nabla f(x^{r-1})
&=
\underbrace{\frac{1}{KN}\sum_{i=1}^N\sum_{k=0}^{K-1} \bigl(\mu_i(y_{i,k}^r)-\mu_i(x^{r-1})\bigr)}_{\text{Local Drift}}
+
\underbrace{\frac{1}{N}\sum_{i=1}^N \bigl(\mu_i(x^{r-1}) - c_i^{r-1}\bigr)}_{\text{Client Control Lag}}
\notag\\
&\quad+
\underbrace{\bigl(c_{\mathrm{srv}}^{r-1} - \nabla f(x^{r-1})\bigr)}_{\text{Server ZNE Error}}.
\end{align}
Applying $\|u+v+w\|^2 \le 3\|u\|^2 + 3\|v\|^2 + 3\|w\|^2$ and taking full expectation:

For the first term, we recall that
$$
\mu_i(x)=\nabla f_i(x)+b_{q,i}(x).
$$
Thus, we have:
\begin{align*}
\|\mu_i(y_{i,k}^r)-\mu_i(x^{r-1})\|^2
&\le
2\|\nabla f_i(y_{i,k}^r)-\nabla f_i(x^{r-1})\|^2
+
2\|b_{q,i}(y_{i,k}^r)-b_{q,i}(x^{r-1})\|^2.
\end{align*}
By $\beta$-smoothness and the $L_q$-Lipschitz property of the bias,
\begin{align*}
\|\mu_i(y_{i,k}^r)-\mu_i(x^{r-1})\|^2
\le
2(\beta^2+L_q^2)\|y_{i,k}^r-x^{r-1}\|^2.
\end{align*}
Averaging over $i$ and $k$ therefore gives
$$
\mathbb E\| \text{Local Drift} \|^2 \le C(\beta^2+L_q^2)E_r.
$$

For the second term, by Jensen's inequality, we have:
\begin{align}
\left\|
\frac{1}{N}\sum_{i=1}^N \bigl(\mu_i(x^{r-1})-c_i^{r-1}\bigr)
\right\|^2
&\le
\frac{1}{N}\sum_{i=1}^N
\|\mu_i(x^{r-1})-c_i^{r-1}\|^2.
\end{align}
Taking expectation and using the definition of $\Gamma_{r-1}$ yields
$$
\mathbb E\| \text{Client Control Lag} \|^2 \le C\Gamma_{r-1}.
$$

For the third term, recall that the server's global reference is $c_{\mathrm{srv}}^{r-1} = \frac{1}{N}\sum_{i=1}^N c_{i,\mathrm{ZNE}}^{r-1}$. We bound its deviation by adding and subtracting the true ZNE mean $\mu_{i,\mathrm{ZNE}}(x^{r-1})$:
\begin{align*}
\mathbb E\left\| c_{\mathrm{srv}}^{r-1} - \nabla f(x^{r-1}) \right\|^2
&\le
2 \mathbb E\left\| \frac{1}{N}\sum_{i=1}^N \bigl(\mu_{i,\mathrm{ZNE}}(x^{r-1}) - \nabla f(x^{r-1})\bigr) \right\|^2 \\
&\quad+
2 \mathbb E\left\| c_{\mathrm{srv}}^{r-1} - \frac{1}{N}\sum_{i=1}^N \mu_{i,\mathrm{ZNE}}(x^{r-1}) \right\|^2.
\end{align*}
By the ZNE assumption, the squared deterministic bias is bounded by $U_q^2/\kappa_b^2$. By Jensen's inequality and the definition of $\Gamma_{r-1}^{\mathrm{ZNE}}$, the second expected squared deviation is exactly bounded by $\Gamma_{r-1}^{\mathrm{ZNE}}$. Thus, the Server ZNE Error is bounded by $2\frac{U_q^2}{\kappa_b^2} + 2\Gamma_{r-1}^{\mathrm{ZNE}}$.

As a result,
\begin{align}
\mathbb E\|h^r-\nabla f(x^{r-1})\|^2
\le
C(\beta^2+L_q^2)E_r + C\Gamma_{r-1} + C\widetilde\eta \Gamma_{r-1}^{\mathrm{ZNE}}
+ C\left( \frac{U_q^2}{\kappa_b^2} \right).
\end{align}
Finally, applying
$$
-\langle a,b\rangle \le -\frac12\|a\|^2+\frac12\|a-b\|^2
$$
with $a=\nabla f(x^{r-1})$ and $b=h^r$ proves the result.
\end{proof}

\paragraph{Step II: second moment.}
The second moment explicitly depends on both the client control lag $\Gamma_{r-1}$ and the server-client mismatch $\Lambda_{r-1}$.

\begin{lemma}[Second moment of the corrected update]
\label{lem:asym_second_moment_final2}
There exists an absolute constant $C>0$ such that
\begin{align}
\mathbb E[\|\Delta x^r\|^2]
\le
C \widetilde\eta^2
\left(
\mathbb E\|\nabla f(x^{r-1})\|^2
+
(\beta^2+L_q^2)E_r
+
\Gamma_{r-1}
+
\Lambda_{r-1}
+
U_q^2
+
\frac{\sigma^2+\sigma_q^2}{KS}
\right).
\end{align}
\end{lemma}

\begin{proof}
First, we define the effective local update direction as:
$$
d_{i,k}^r
:=
\widetilde g_i(y_{i,k}^r;\xi_{i,k}^r)-c_i^{r-1}+c_{\mathrm{srv}}^{r-1}.
$$

By definition, the global update is $\Delta x^r = -\widetilde\eta \frac{1}{SK}\sum_{i\in\mathcal S_r}\sum_{k=0}^{K-1} d_{i,k}^r$.

We decompose $d_{i,k}^r$ into zero-mean stochastic noise and structural error terms:
\begin{align}
d_{i,k}^r
&=
\underbrace{\mu(x^{r-1})}_{\text{Anchor}}
+
\underbrace{\bigl(\mu_i(y_{i,k}^r)-\mu_i(x^{r-1})\bigr)}_{\text{Local Drift}}
+
\underbrace{\bigl(\mu_i(x^{r-1})-c_i^{r-1}\bigr)}_{\text{Client Lag}}
\notag\\
&\quad+
\underbrace{\left(c_{\mathrm{srv}}^{r-1}-\frac1N\sum_{j=1}^N c_j^{r-1}\right)}_{\text{Server Mismatch}}
+
\underbrace{\left(\frac1N\sum_{j=1}^N c_j^{r-1}-\mu(x^{r-1})\right)}_{\text{Population Lag}}
+
\underbrace{\zeta_{i,k}^r}_{\text{Noise}},
\end{align}
where
$$
\mu_i(x):=\nabla f_i(x)+b_{q,i}(x),
\qquad
\mu(x):=\frac1N\sum_{i=1}^N\mu_i(x),
\qquad
\zeta_{i,k}^r:=\widetilde g_i(y_{i,k}^r;\xi_{i,k}^r)-\mu_i(y_{i,k}^r).
$$

To bound $\mathbb E\|\Delta x^r\|^2$, we evaluate the expected squared norm of the averaged $d_{i,k}^r$. We first isolate the zero-mean stochastic noise. Because the fresh local noise $\zeta_{i,k}^r$ is independent across the $S$ clients and $K$ local steps, its cross-terms vanish in expectation. The variance averages over the $SK$ samples, contributing exactly $\frac{\sigma^2+\sigma_q^2}{SK}$ to the final bound.

For the remaining five structural terms, we apply the expansion inequality $\|\sum_{m=1}^5 u_m\|^2 \le 5\sum_{m=1}^5 \|u_m\|^2$. Taking the expectation and averaging over the subset $\mathcal{S}_r$ and $K$ steps, we bound the squared norm of each constituent term as follows:

1. \textbf{Anchor:} Bounded by the true gradient and hardware bias limit: $$\mathbb E\|\mu(x^{r-1})\|^2 \le 2\mathbb E\|\nabla f(x^{r-1})\|^2+2U_q^2$$.

2. \textbf{Local Drift:} By the $\beta$-smoothness of the loss and the $L_q$-Lipschitz property of the bias, this term is bounded by $2(\beta^2+L_q^2)E_r$.

3. \textbf{Client Lag:} By the properties of uniform sampling without replacement, the expected squared norm of the subset average is bounded by the population average. Thus, $\mathbb E\| \frac{1}{S}\sum_{i\in\mathcal S_r} (\mu_i(x^{r-1})-c_i^{r-1}) \|^2 \le \frac{1}{N}\sum_{i=1}^N \mathbb E\|\mu_i(x^{r-1})-c_i^{r-1}\|^2 = \Gamma_{r-1}$.

4. \textbf{Server Mismatch:} Independent of the subset sampling, this corresponds exactly to the definition of $\Lambda_{r-1}$.

5. \textbf{Population Lag:} By Jensen's inequality and the definition of $\Gamma_{r-1}$, we have $\mathbb E\| \frac1N\sum_{j=1}^N (c_j^{r-1}-\mu_j(x^{r-1})) \|^2 \le \Gamma_{r-1}$.

Collecting the isolated noise variance and the five structural bounds yields the stated bound.
\end{proof}

\paragraph{Step III: local drift.}
We now bound how far the clients drift from the global anchor point during their $K$ local steps. Because the clients utilize a frozen local control variate, this drift is driven by the staleness of their correction compasses, the baseline quantum noise, and the hardware bias.

\begin{lemma}[Local-drift bound]
\label{lem:asym_drift_final2}
Assume the global and local step sizes satisfy
\begin{align}
\widetilde\eta \le \frac{c_1}{\sqrt{\beta^2+L_q^2}}\frac{S}{N},
\qquad
\eta_g\ge 1.
\label{eq:asym_steps_final2}
\end{align}
Then there exists an absolute constant $C>0$ such that the population local-drift energy is bounded by
\begin{align*}
E_r
\le
C\eta_\ell^2K^2
\left(
\mathbb E\|\nabla f(x^{r-1})\|^2
+
\Gamma_{r-1}
+
\Lambda_{r-1}
+
U_q^2
+
\sigma^2+\sigma_q^2
\right).
\end{align*}
\end{lemma}

\begin{proof}
For any client $i$, we define the expected squared distance from the anchor at step $k$:
$$
D_{i,k}^r:=\mathbb E\|y_{i,k}^r-x^{r-1}\|^2,
\qquad
D_{i,0}^r=0.
$$
Unrolling the local update rule \eqref{eq:asym_scaffold_local_final2} gives
$$
y_{i,k}^r-x^{r-1}
=
-\eta_\ell\sum_{j=0}^{k-1} d_{i,j}^r,
$$
where $d_{i,j}^r := \widetilde g_i(y_{i,j}^r;\xi_{i,j}^r)-c_i^{r-1}+c_{\mathrm{srv}}^{r-1}$ is the effective step direction. Applying the standard expansion $\|\sum_{j=0}^{k-1} v_j\|^2 \le K\sum_{j=0}^{k-1}\|v_j\|^2$, we obtain:
\begin{align}
D_{i,k}^r
&\le
\eta_\ell^2 K \sum_{j=0}^{k-1} \mathbb E\|d_{i,j}^r\|^2.
\label{eq:asym_drift_start_short}
\end{align}

To derive the bounds for $\mathbb E\|d_{i,j}^r\|^2$, we recall the exact decomposition established in Lemma~\ref{lem:asym_second_moment_final2}. Applying the identical expansion inequality to the constituent structural terms, we bound them exactly as before. The single distinction is that we evaluate the zero-mean stochastic noise for a single sample rather than averaging over $SK$ steps, which directly yields a variance contribution of $\sigma^2+\sigma_q^2$. Grouping the terms yields:
\begin{align*}
\mathbb E\|d_{i,j}^r\|^2
&\le
C(\beta^2+L_q^2)\|y_{i,j}^r-x^{r-1}\|^2
+
C
\left(
\mathbb E\|\nabla f(x^{r-1})\|^2
+
\Gamma_{r-1}
+
\Lambda_{r-1}
+
U_q^2
+
\sigma^2+\sigma_q^2
\right).
\end{align*}

Substituting this single-step bound back into \eqref{eq:asym_drift_start_short} yields:
\begin{align}
D_{i,k}^r
&\le
C(\beta^2+L_q^2)\eta_\ell^2 K \sum_{j=0}^{k-1} D_{i,j}^r
\notag\\
&\quad+
C\eta_\ell^2 K^2
\left(
\mathbb E\|\nabla f(x^{r-1})\|^2
+
\Gamma_{r-1}
+
\Lambda_{r-1}
+
U_q^2
+
\sigma^2+\sigma_q^2
\right).
\label{eq:asym_drift_rec_short}
\end{align}

Under the step-size condition \eqref{eq:asym_steps_final2}, the recursive contraction factor $(\beta^2+L_q^2)\eta_\ell^2 K^2 = (\beta^2+L_q^2)\frac{\widetilde\eta^2}{\eta_g^2}$ is sufficiently small. Applying a discrete Grönwall argument to resolve the recursion, we have:
\begin{align}
D_{i,k}^r
\le
C\eta_\ell^2 K^2
\left(
\mathbb E\|\nabla f(x^{r-1})\|^2
+
\Gamma_{r-1}
+
\Lambda_{r-1}
+
U_q^2
+
\sigma^2+\sigma_q^2
\right),
\end{align}
holding uniformly for all $k \le K$. Averaging over all client $i$ and steps $k$ yields the final bound.
\end{proof}

\paragraph{Step IV: recursive control errors and mismatch bound.}
Here we bound the stateful staleness of both the unmitigated controls ($\Gamma_r$), the ZNE-corrected controls ($\Gamma_r^{\mathrm{ZNE}}$), and the server-client mismatch ($\Lambda_r$).

\begin{lemma}[Client-control error recursion]
\label{lem:asym_gamma_final2}
Under the step-size condition \eqref{eq:asym_steps_final2}, there exist constants $0<\rho<1$ and $C>0$ such that
\begin{align*}
\Gamma_r
\le
\left(1-\rho\alpha\frac{S}{N}\right)\Gamma_{r-1}
+
C \alpha^2 \frac{S}{N}
\left(
\mathbb E\|\nabla f(x^{r-1})\|^2
+
(\beta^2+L_q^2)E_r
+
\Lambda_{r-1}
+
U_q^2
+
\sigma^2+\sigma_q^2
\right).
\end{align*}
\end{lemma}

\begin{proof}

We recall the definition of the control-estimation error $\Gamma_r := \frac{1}{N}\sum_{i=1}^N \mathbb E\|c_i^r-\mu_i(x^r)\|^2$. We partition the population into sampled clients ($\mathcal S_r$) and unsampled clients. Let $\widetilde c_i^r$ denote the updated control variate before the anchor shifts from $x^{r-1}$ to $x^r$.

For sampled clients $i\in\mathcal S_r$, the memory is updated via momentum. We subtract the true anchor expectation $\mu_i(x^{r-1})$:
$$
\widetilde c_i^r - \mu_i(x^{r-1}) = (1-\alpha)(c_i^{r-1} - \mu_i(x^{r-1})) + \alpha (\widetilde g_i(x^{r-1};\zeta_i^r) - \mu_i(x^{r-1})).
$$
Because the new stochastic shot $\zeta_i^r$ is conditionally independent given $\mathcal F_{r-1}$, the cross-term vanishes in expectation:
$$
\mathbb E\|\widetilde c_i^r - \mu_i(x^{r-1})\|^2 = (1-\alpha)^2 \mathbb E\|c_i^{r-1} - \mu_i(x^{r-1})\|^2 + \alpha^2 (\sigma^2+\sigma_q^2).
$$

For unsampled clients $i \notin \mathcal S_r$, the memory is strictly frozen ($\widetilde c_i^r = c_i^{r-1}$), so the expected squared error remains exactly $\mathbb E\|c_i^{r-1} - \mu_i(x^{r-1})\|^2$.

Averaging over the population (sampled fraction $S/N$ and unsampled $1-S/N$), we define the intermediate population error $\widetilde \Gamma_r$ evaluated at the old anchor:
\begin{align*}
\widetilde \Gamma_r
&= \frac{S}{N} \left[ (1-\alpha)^2 \Gamma_{r-1} + \alpha^2 (\sigma^2+\sigma_q^2) \right] + \left(1-\frac{S}{N}\right) \Gamma_{r-1} \\
&= \left( 1 - \frac{S}{N}(2\alpha - \alpha^2) \right) \Gamma_{r-1} + \alpha^2 \frac{S}{N} (\sigma^2+\sigma_q^2).
\end{align*}
Since the momentum parameter $\alpha \in (0, 1]$, we strictly have $2\alpha - \alpha^2 \ge \alpha$. Thus, the population contraction is strictly bounded by $1 - \alpha\frac{S}{N}$.

Finally, we shift the evaluation point from $x^{r-1}$ to the new anchor $x^r$. Applying Young's inequality $\|A+B\|^2 \le (1+\gamma)\|A\|^2 + (1+1/\gamma)\|B\|^2$ for some $\gamma > 0$, and using $\beta$-smoothness alongside the $L_q$-Lipschitz bias property:
\begin{align*}
\Gamma_r
&\le (1+\gamma) \widetilde \Gamma_r + \left(1+\frac{1}{\gamma}\right) \frac{1}{N}\sum_{i=1}^N \mathbb E\|\mu_i(x^{r-1}) - \mu_i(x^r)\|^2 \\
&\le (1+\gamma) \left[ \left(1 - \alpha\frac{S}{N}\right) \Gamma_{r-1} + \alpha^2 \frac{S}{N} (\sigma^2+\sigma_q^2) \right] + C\left(1+\frac{1}{\gamma}\right)(\beta^2+L_q^2)\mathbb E\|\Delta x^r\|^2.
\end{align*}
Choosing $\gamma = \frac{\alpha S}{2N}$ ensures that $(1+\gamma)(1-\alpha\frac{S}{N}) \le 1 - \frac{\alpha S}{2N}$. Setting $\rho = 1/2$, this yields a strict contraction factor of $(1 - \rho\alpha\frac{S}{N})$. The multiplier on the global step distance is bounded by $\mathcal{O}(\frac{N}{\alpha S})$. Substituting $\mathbb E\|\Delta x^r\|^2$ from Lemma~\ref{lem:asym_second_moment_final2} and absorbing constants into $C$ yields the final stated bound.

\end{proof}

\begin{lemma}[ZNE server-control error recursion]
\label{lem:asym_gamma_zne_final2}
Under the step-size condition \eqref{eq:asym_steps_final2}, there exist constants $0<\rho<1$ and $C>0$ such that
\begin{align*}
\Gamma_r^{\mathrm{ZNE}}
\le
\left(1-\rho\alpha\frac{S}{N}\right)\Gamma_{r-1}^{\mathrm{ZNE}}
+
C \alpha^2 \frac{S}{N}
\left(
\mathbb E\|\nabla f(x^{r-1})\|^2
+
(\beta^2+L_q^2)E_r
+
\Lambda_{r-1}
+
U_q^2
+
\sigma^2+\kappa_v\sigma_q^2
\right).
\end{align*}
\end{lemma}

\begin{proof}
Applying the identical algebraic decomposition, variance isolation, and Young's inequality shifting argument established in the proof of Lemma~\ref{lem:asym_gamma_final2}, we have this stated bound.
\end{proof}

\begin{lemma}[Server-client control mismatch bound]
\label{lem:asym_lambda_final2}
There exists an absolute constant $C>0$ such that
\begin{align}
\Lambda_{r-1}
\le
C
\left(
\Gamma_{r-1}
+
\Gamma_{r-1}^{\mathrm{ZNE}}
+
\frac{U_q^2}{\kappa_b^2}
+
U_q^2
\right).
\end{align}
\end{lemma}

\begin{proof}
We recall the definition of the server-client control mismatch:
$$
\Lambda_{r-1}
=
\mathbb E\!\left[
\left\|
c_{\mathrm{srv}}^{r-1}-\frac{1}{N}\sum_{i=1}^N c_i^{r-1}
\right\|^2
\right].
$$
Because the server's running accumulator mathematically tracks the population average of the stored ZNE oracles, we have $c_{\mathrm{srv}}^{r-1} = \frac{1}{N}\sum_{i=1}^N c_{i,\mathrm{ZNE}}^{r-1}$. We can therefore directly combine the sums without any subset-sampling mismatch:
\begin{align*}
c_{\mathrm{srv}}^{r-1}-\frac1N\sum_{i=1}^N c_i^{r-1}
&=
\frac1N\sum_{i=1}^N \bigl( c_{i,\mathrm{ZNE}}^{r-1} - c_i^{r-1} \bigr).
\end{align*}

We add and subtract the expected values of the respective local oracles, $\mu_{i,\mathrm{ZNE}}(x^{r-1})$ and $\mu_i(x^{r-1})$. Because both oracles are evaluated by the same client on the same data, the true gradients $\nabla f_i(x^{r-1})$ perfectly cancel:
\begin{align}
c_{\mathrm{srv}}^{r-1}-\frac1N\sum_{i=1}^N c_i^{r-1}
&=
\underbrace{ \frac1N\sum_{i=1}^N \bigl(c_{i,\mathrm{ZNE}}^{r-1} - \mu_{i,\mathrm{ZNE}}(x^{r-1})\bigr) }_{\text{ZNE Stochastic Noise}}
\notag\\
&\quad+
\underbrace{ \frac1N\sum_{i=1}^N \bigl(b_{q,i}^{\mathrm{ZNE}}(x^{r-1}) - b_{q,i}(x^{r-1})\bigr) }_{\text{Hardware Bias Difference}}
\notag\\
&\quad+
\underbrace{ \frac1N\sum_{i=1}^N \bigl(\mu_i(x^{r-1}) - c_i^{r-1}\bigr) }_{\text{Client Control Lag}}.
\end{align}

Applying $\|u+v+w\|^2 \le 3\|u\|^2 + 3\|v\|^2 + 3\|w\|^2$, we bound the expected squared norm of each term:

1. \textbf{ZNE Control Lag:} By definition and Jensen's inequality, the expected squared lag of the population's ZNE-corrected controls against their true anchor expectations is bounded by $\Gamma_{r-1}^{\mathrm{ZNE}}$.

2. \textbf{Hardware Bias Difference:} Applying Jensen's inequality and the bounded bias assumptions, the difference in the absolute hardware biases is bounded by $2(U_q^2/\kappa_b^2) + 2U_q^2$.

3. \textbf{Client Control Lag:} By definition and Jensen's inequality, the expected squared lag of the population's biased controls against their true anchor expectations is exactly $\Gamma_{r-1}$.

Summing these bounds and absorbing the numerical multipliers into an constant $C$ proves the claim.
\end{proof}

\paragraph{Step V: Lyapunov descent.}
Because both the client memory and the server ZNE memory inject first-order tracking errors into the objective, we define a Dual-Lyapunov function to force contraction across both variables simultaneously:
\begin{align}
\Phi_r := \mathbb E[f(x^r)] + \lambda \Gamma_r + \lambda_{\mathrm{ZNE}} \Gamma_r^{\mathrm{ZNE}},
\qquad
\lambda, \lambda_{\mathrm{ZNE}} = C_\lambda \widetilde\eta\frac{N}{\alpha S}.
\end{align}

Because both the client memory and the server ZNE memory inject first-order tracking errors into the objective, we define a Dual-Lyapunov function to force contraction across both variables simultaneously:
\begin{align}
\Phi_r := \mathbb E[f(x^r)] + \lambda \Gamma_r + \lambda_{\mathrm{ZNE}} \Gamma_r^{\mathrm{ZNE}},
\qquad
\lambda, \lambda_{\mathrm{ZNE}} = C_\lambda \widetilde\eta\frac{N}{\alpha S}.
\end{align}

\begin{lemma}[One-round Dual-Lyapunov descent]
\label{lem:asym_lyapunov_final2}
Under the step-size condition \eqref{eq:asym_steps_final2}, there exist absolute constants $c, C > 0$ such that
\begin{align*}
\Phi_r
\le
\Phi_{r-1}
-
c \widetilde\eta\,\mathbb E\|\nabla f(x^{r-1})\|^2
+
C \widetilde\eta
\left(
\widetilde\eta U_q^2
+
\frac{U_q^2}{\kappa_b^2}
+
\alpha(\sigma^2+\kappa_v\sigma_q^2)
\right).
\end{align*}
\end{lemma}

\begin{proof}
Substitute the descent direction (Lemma~\ref{lem:asym_direction_final2}), second moment (Lemma~\ref{lem:asym_second_moment_final2}), and local drift (Lemma~\ref{lem:asym_drift_final2}) bounds into the smoothness inequality \eqref{eq:asym_smooth_final2}. Note that the first-order descent injects $\mathcal{O}(\widetilde\eta \Gamma_{r-1})$ and $\mathcal{O}(\widetilde\eta \Gamma_{r-1}^{\mathrm{ZNE}})$ errors, while the server-client mismatch $\Lambda_{r-1}$ appears strictly through the second-order variance and drift penalties, meaning it is scaled by $\widetilde\eta^2$. Using the generic positive constant $C$ and strict descent constant $c$:
\begin{align*}
\mathbb E[f(x^r)]
\le\;&
\mathbb E[f(x^{r-1})]
-
c\widetilde\eta\,\mathbb E\|\nabla f(x^{r-1})\|^2
+
C\widetilde\eta \Gamma_{r-1}
+
C\widetilde\eta \Gamma_{r-1}^{\mathrm{ZNE}}
+
C\widetilde\eta^2\Lambda_{r-1}
\notag\\
&+
C\widetilde\eta \frac{U_q^2}{\kappa_b^2}
+
C\widetilde\eta^2 U_q^2
+
C\widetilde\eta^2\frac{\sigma^2+\sigma_q^2}{KS}.
\end{align*}

Next, we expand the server-client mismatch $\Lambda_{r-1}$ using the bound from Lemma~\ref{lem:asym_lambda_final2}. Crucially, because the stateful tracking architecture perfectly synchronizes the population, no $G^2$ heterogeneity penalty is generated. Expanding $\Lambda_{r-1}$ yields terms bounded by $\Gamma_{r-1}$, $\Gamma_{r-1}^{\mathrm{ZNE}}$, $U_q^2/\kappa_b^2$, and $U_q^2$. Because this entire expansion is multiplied by $\widetilde\eta^2$, the tracking errors and bias residuals are strictly dominated by their first-order counterparts ($\mathcal{O}(\widetilde\eta)$) already present in the objective bound, allowing them to be absorbed into the absolute constant $C$. The only term that lacks a first-order counterpart is $C\widetilde\eta^2 U_q^2$, so it is still surviving. 

To form the Lyapunov bound, we add the scaled memory recursions $\lambda \Gamma_r$ and $\lambda_{\mathrm{ZNE}} \Gamma_r^{\mathrm{ZNE}}$ (from Lemma~\ref{lem:asym_gamma_final2} and Lemma~\ref{lem:asym_gamma_zne_final2}) to the objective descent, grouping the terms by the definition $\Phi_r := \mathbb E[f(x^r)] + \lambda \Gamma_r + \lambda_{\mathrm{ZNE}} \Gamma_r^{\mathrm{ZNE}}$:
\begin{align}
\Phi_r
&\le
\Phi_{r-1}
-
c\widetilde\eta\,\mathbb E\|\nabla f(x^{r-1})\|^2
+
\underbrace{\left(C\widetilde\eta - \lambda \rho \alpha \frac{S}{N}\right)}_{\text{Net } \Gamma_{r-1} \text{ Coefficient}} \Gamma_{r-1}
+
\underbrace{\left(C\widetilde\eta - \lambda_{\mathrm{ZNE}} \rho \alpha \frac{S}{N}\right)}_{\text{Net } \Gamma_{r-1}^{\mathrm{ZNE}} \text{ Coefficient}} \Gamma_{r-1}^{\mathrm{ZNE}}
\notag\\
&\quad+
\lambda \left( C \alpha^2 \frac{S}{N}(\sigma^2+\sigma_q^2) \right)
+
\lambda_{\mathrm{ZNE}} \left( C \alpha^2 \frac{S}{N}(\sigma^2+\kappa_v\sigma_q^2) \right)
\notag\\
&\quad+
C\widetilde\eta^2 U_q^2
+
C\widetilde\eta \frac{U_q^2}{\kappa_b^2}.
\end{align}

We now substitute the defined Lyapunov weights $\lambda, \lambda_{\mathrm{ZNE}} = C_\lambda \widetilde\eta \frac{N}{\alpha S}$. Both net tracking coefficients evaluate to $(C - C_\lambda \rho)\widetilde\eta$. By choosing $C_\lambda \ge C/\rho$, these terms are strictly non-positive and drop from the bound. 

Finally, we evaluate the injected memory variance from both updates using our Lyapunov weights:

\begin{align*}
&\lambda \left( C \alpha^2 \frac{S}{N}(\sigma^2+\sigma_q^2) \right)
+
\lambda_{\mathrm{ZNE}} \left( C \alpha^2 \frac{S}{N}(\sigma^2+\kappa_v\sigma_q^2) \right) \\
&=
\left(C_\lambda \widetilde\eta \frac{N}{\alpha S}\right) C \alpha^2 \frac{S}{N} \Big[ (\sigma^2+\sigma_q^2) + (\sigma^2+\kappa_v\sigma_q^2) \Big] \\
&=
C'\widetilde\eta \alpha (\sigma^2+\sigma_q^2) + C'\widetilde\eta \alpha (\sigma^2+\kappa_v\sigma_q^2).
\end{align*}

Absorbing the multiplied constants back into the generic absolute constant $C$, we collect the surviving components: the negative descent term, this dual injected variance, the step-size-dependent bias residual $\widetilde\eta^2 U_q^2$, and the mitigated ZNE hardware bias $\frac{U_q^2}{\kappa_b^2}$. Summing these components, we have the stated bound.
\end{proof}

\paragraph{Step VI: final theorem.}

\begin{theorem}[Convergence of \NAMEA]
\label{thm:asym_zne_final2}
Under Assumptions~\ref{ass:smooth}--\ref{ass:lipschitz_bias} and the step-size condition \eqref{eq:asym_steps_final2}. Assume the dual control variates are updated with momentum $\alpha = \mathcal{O}\left(\widetilde\eta \frac{N}{S}\right)$. Let $F:=\Phi_0-f^\star$. Run the method for $R$ rounds, and let $\bar x^R$ be chosen uniformly at random from the iterates $\{x^0,\dots,x^{R-1}\}$. Then, the expected gradient norm is bounded by:
\begin{align*}
\mathbb E\|\nabla f(\bar x^R)\|^2
\le\;&
\mathcal O\!\left(\frac{F}{\widetilde\eta R}\right)
+
 \mathcal O\!\left( \widetilde\eta \frac{N}{S} (\sigma^2+\kappa_v\sigma_q^2) \right) 
+
\mathcal O\!\left(
\frac{U_q^2}{\kappa_b^2}
+
\widetilde\eta U_q^2
\right).
\end{align*}
\end{theorem}

\begin{proof}
We begin with the one-round Lyapunov descent inequality from Lemma~\ref{lem:asym_lyapunov_final2}. Rearranging to isolate the expected gradient norm, we have:
\begin{align*}
c \widetilde\eta\,\mathbb E\|\nabla f(x^{r-1})\|^2 \le \Phi_{r-1} - \Phi_r + C \widetilde\eta\left(\widetilde\eta U_q^2+\frac{U_q^2}{\kappa_b^2}+\alpha (\sigma^2+\kappa_v\sigma_q^2)\right).
\end{align*}

Summing this inequality over $r=1,\dots,R$ creates a telescoping sum for the Lyapunov function: $\sum_{r=1}^R (\Phi_{r-1} - \Phi_r) = \Phi_0 - \Phi_R$. Since $f$ is globally bounded below by $f^\star$ and the tracking errors $\Gamma_R, \Gamma_R^{\mathrm{ZNE}} \ge 0$, we strictly have $\Phi_R \ge f^\star$. Dividing the aggregated inequality by $R c \widetilde\eta$ and substituting $\alpha = \mathcal{O}(\widetilde\eta \frac{N}{S})$, we have:

\begin{align*}
\frac{1}{R}\sum_{r=1}^R \mathbb E\|\nabla f(x^{r-1})\|^2 \le \frac{\Phi_0 - f^\star}{c \widetilde\eta R} + \frac{C}{c}\left(\widetilde\eta U_q^2+\frac{U_q^2}{\kappa_b^2}+ \widetilde\eta \frac{N}{S} (\sigma^2+\kappa_v\sigma_q^2)\right).
\end{align*}

Because the output $\bar x^R$ is sampled uniformly at random from the $R$ prior iterates, its expected gradient norm is exactly equal to the left side of the equation. Absorbing the absolute constants $C$ and $c$ into the $\mathcal{O}(\cdot)$ notation yields the stated theorem.
\end{proof}

The resulting theorem highlights the structural benefits of the \NAMEA\ architecture in solving the double-drifting problem in QFL. First of all, the classical data heterogeneity penalty $\mathcal{O}(G^2/S)$ is mathematically eliminated by the server's stateful tracking.

In addition, the bound perfectly isolates the bias-variance trade-off of quantum hardware mitigation. The raw, unmitigated quantum bias is relegated to a step-size-weighted residual $\mathcal{O}(\widetilde\eta U_q^2)$, which strictly vanishes as the learning rate decays. The remaining hardware bias is dictated entirely by the ZNE residual $\mathcal{O}(U_q^2/\kappa_b^2)$, which effectively vanishes if the zero-noise extrapolation is sufficiently accurate (large $\kappa_b$).

While the fundamental physical trade-off for this bias suppression is variance amplification ($\kappa_v$), this penalty is successfully neutralized by \NAMEA's stateful momentum tracking. Because the dual memory updates are governed by an exponential moving average (EMA) whose momentum parameter is mathematically coupled to both the effective learning rate and the system participation ratio ($\alpha \propto \widetilde\eta \frac{N}{S}$), the injected amplified variance is forced to scale as $\mathcal{O}(\widetilde\eta \frac{N}{S}(\sigma^2+\kappa_v\sigma_q^2))$. This demonstrates a profound property of \NAMEA. It systematically offsets both the physical variance penalty of quantum error mitigation and the statistical variance of offline clients.

\section{Description of Datasets \& Federated Settings.}

\paragraph{Binary Blobs \citep{bowles2024better}:} This dataset was designed as a binary equivalent to the standard 'Gaussian blobs' datasets, which typically consist of real-valued vectors. The generation process begins with eight predefined 16-bit reference strings (illustrated in the top row of Figure \ref{fig:binaryblobs}). For each sample, one of these eight base bitstrings is selected at random, and then independent bit-flip noise is applied to each position with a probability of $p = 0.05$.  This procedure results in a probability distribution over the 16-bit space characterized by eight distinct, visually recognizable modes. Using this methodology, a training set of 5,000 samples and a testing set of 10,000 samples is generated. Representative samples are displayed in the ‘True’ row of Figure \ref{fig:binaryblobs}.  

\begin{figure}[h!]
    \centering
    \includegraphics[width=\linewidth]{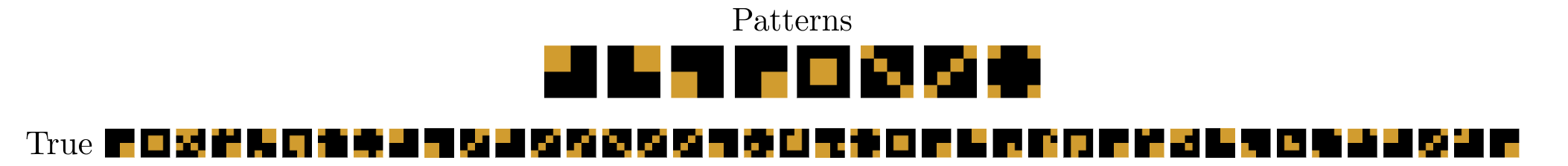}
    \caption{Visualization of the eight base images used to generate the Binary Blobs dataset, compared with samples drawn from the ground truth distribution and the outputs of each trained model. Dark squares denote a value of 1 in the corresponding bitstring. Figure adapted from \citep{bowles2024better}.}
    \label{fig:binaryblobs}
\end{figure}
\paragraph{The selection of QML architecture:} In the QML architecture, the variational quantum circuit uses 4 qubits and consists of an amplitude embedding layer to encode classical inputs, followed by 5 strongly entangling layers. Measurement is implemented via computational-basis projective measurement to produce class probabilities. This architecture is widely adopted in prior QML studies (\citep{Havlicek2019_QML,du2021quantum, watkins2023quantum,long2025hybrid}). 

\paragraph{Implementation Details:} The implementation uses the PennyLane library version 0.41.1 \citep{bergholm2018pennylane} and PyTorch 2.8.0. All experiments were conducted on a Linux workstation running Ubuntu 20.04 LTS, 
equipped with an Intel(R) Xeon(R) CPU E5-2697 v4 @ 2.30GHz (18 cores 36 threads), 
384GB RAM, and two NVIDIA RTX A6000 GPU (48GB VRAM each). Our CUDA version is 12.6. 
\paragraph{Non-IID FL Settings:} We partition the dataset across 8 clients using a Dirichlet distribution ($\alpha = 0.3$) to induce statistical data heterogeneity. Detailed clients distributions are given in Fig. \ref{fig:client_distribution}.

\begin{figure}[h!]
    \centering
    \includegraphics[width=1\linewidth]{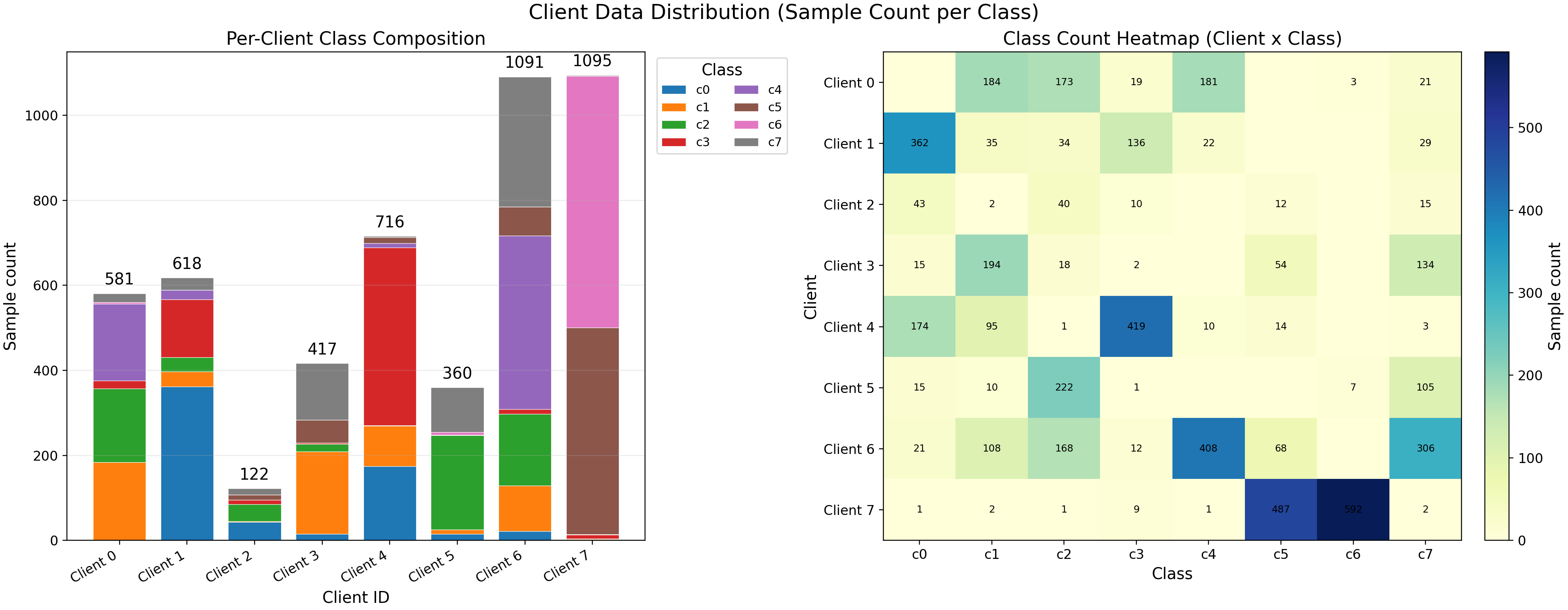}
    \caption{Non-IID client distribution (Dirichlet $\alpha = 0.3$)}
    \label{fig:client_distribution}
\end{figure}
\label{appendix:detailsdatasets}

\section{Performance of \NAMEA\ under hardware-induced bias $U_q$}\label{appx:moreperf}
Under larger levels of depolarizing noise \(p \in \{0.02, 0.03\}\), the performance gap between \NAMEA\ and the baselines becomes even more pronounced. As shown in Fig.~\ref{fig:depo002} and Fig.~\ref{fig:depo003}, increasing the hardware noise level substantially degrades the optimization behavior of FedAvg and SCAFFOLD, leading to higher training and test loss, lower train and test accuracy, and noticeably larger fluctuations across communication rounds. This trend is consistent with the theoretical analysis that hardware-induced bias introduces a persistent error floor of order \(U_q^2\) for FedAvg, while SCAFFOLD can only compensate for classical client drift and does not explicitly remove the systematic bias arising from noisy quantum gradients.

In contrast, \NAMEA\ remains significantly more stable even when the depolarizing strength increases. Although all methods experience some deterioration under stronger noise, \NAMEA\ consistently preserves better convergence speed and stronger final performance than both FedAvg and SCAFFOLD in all reported metrics. These results support the main claim of the paper: as hardware noise becomes more severe, explicitly correcting the quantum bias becomes increasingly important, and the ZNE-anchored correction in \NAMEA\ provides a clear advantage over classical federated aggregation strategies.
\begin{figure}[h!]
    \centering
    \includegraphics[width=\linewidth]{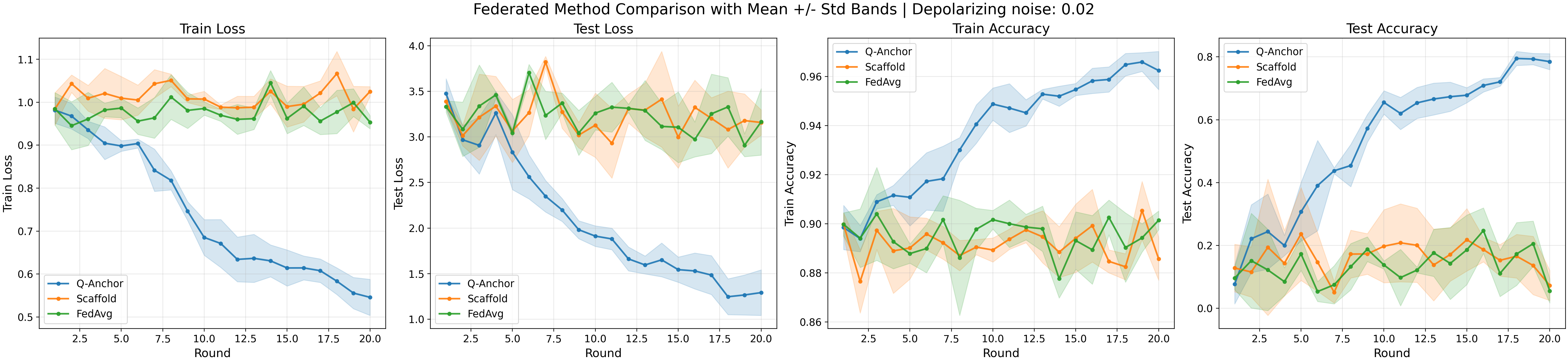}
    \caption{Performance comparison of \NAMEA\ and other FL algorithms under depolarizing noise $p=0.02$ with analytic gradients.}
    \label{fig:depo002}
\end{figure}

\begin{figure}[h!]
    \centering
    \includegraphics[width=\linewidth]{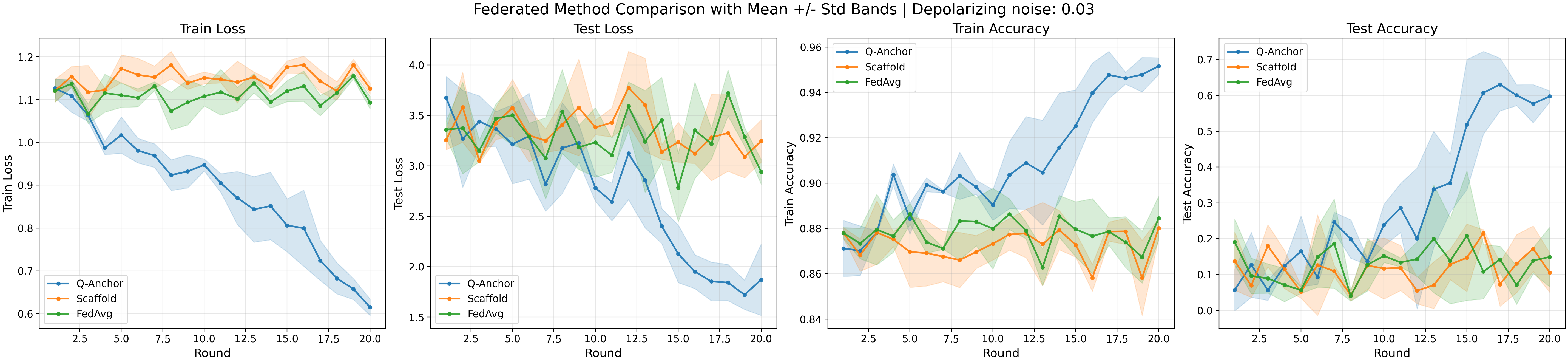}
    \caption{Performance comparison of \NAMEA\ and other FL algorithms under depolarizing noise $p=0.03$ with analytic gradients.}
    \label{fig:depo003}
\end{figure}

\section{Performance of \NAMEA\ under varying finite-shot variance $\sigma_q^2$}\label{appendix:shotnoise}

The precision of gradient estimation in Variational Quantum Algorithms (VQAs) is inherently tied to the number of measurement shots used to evaluate quantum observables. As shown in Figure~\ref{fig:variancevsshots}, there is a clear trade-off between statistical precision and computational overhead.

\begin{figure}[h!]
    \centering
    \includegraphics[width=0.5\linewidth]{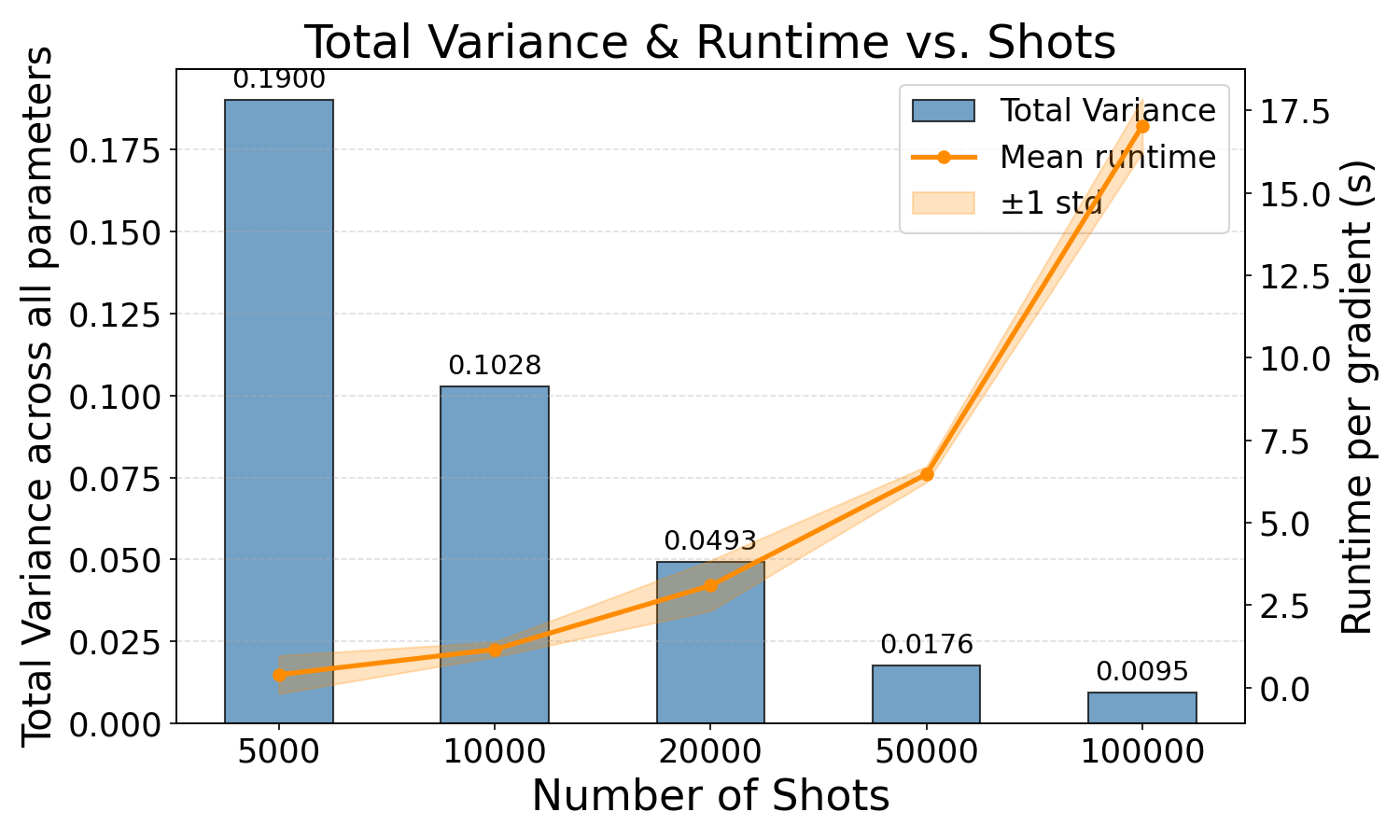}
    \caption{Impact of finite measurement shots on statistical variance and runtime.}
    \label{fig:variancevsshots}
\end{figure}

The empirical data in Figure~\ref{fig:variancevsshots} highlights two primary trends:
\begin{itemize}
    \item \textbf{Variance Decay:} The total variance across all parameters decreases sharply as the shot count increases. Increasing the shots from $5,000$ to $100,000$ results in a variance reduction of approximately $95\%$ (from $0.1900$ down to $0.0095$). This follows the expected $1/N$ scaling where $N$ is the number of shots.
    \item \textbf{Runtime Scaling:} This gain in precision comes at a direct cost to execution time. The runtime per gradient evaluation scales from approximately $0.4$s at $5,000$ shots to over $17$s at $100,000$ shots. Runtime is measured on a single core of Intel(R) Xeon(R) CPU E5-2697 v4 @ 2.30GHz.
\end{itemize}

\begin{figure}[h!]
    \centering
    \includegraphics[width=\linewidth]{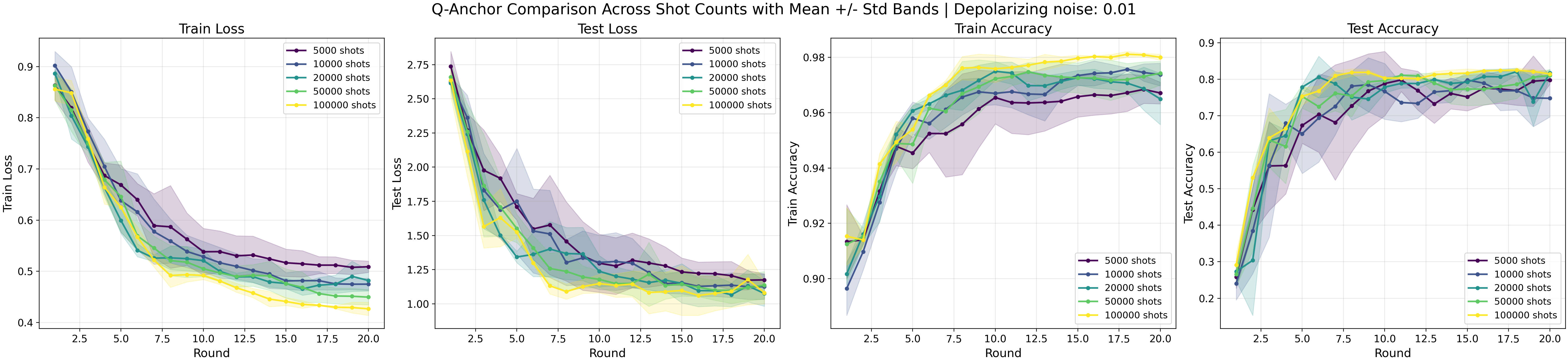}
    \caption{Impact of number of shots on \NAMEA\ training dynamics under depolarizing noise $p = 0.01$.}
    \label{fig:compareshots}
\end{figure}

Figure~\ref{fig:compareshots} illustrates how this statistical variance $\sigma_q^2$ translates to training performance for \NAMEA. Under a constant depolarizing noise level of $0.01$, several key observations can be made:

\paragraph{Training Stability and Convergence}
As the number of shots increases, the shaded standard deviation bands noticeably narrow. In the $5,000$-shot regime (purple curves), the training and test loss exhibit high volatility and slower initial descent. In contrast, the $50,000$ and $100,000$-shot configurations (yellow and bright green) produce significantly smoother trajectories, suggesting that lower variance allows the optimizer to follow the true gradient more closely.

\paragraph{Robustness to Shot Noise}
Despite the high variance in low-shot settings, \NAMEA\ demonstrates remarkable robustness. Even with only $5,000$ shots, the model eventually converges to a test accuracy (approx. $80\%$) that is competitive with the $100,000$-shot configuration. This indicates that the \NAMEA\ framework can tolerate non-trivial shot noise without collapsing.

\paragraph{Diminishing Returns}
The performance gap between $20,000$ shots and $100,000$ shots is marginal in terms of final accuracy. Given the linear increase in runtime, these results suggest that $20,000$ shots provide an optimal balance for this task, offering sufficient variance reduction to stabilize training without the prohibitive overhead of higher shot counts.
\end{document}